\documentclass[letterpaper]{article} 
\usepackage{aaai2026}  
\usepackage{times}  
\usepackage{helvet}  
\usepackage{courier}  
\usepackage[hyphens]{url}  
\usepackage{graphicx} 
\usepackage{natbib}  
\usepackage{caption} 
\usepackage[ruled,vlined,linesnumbered]{algorithm2e}
\usepackage{amsmath}
\usepackage{amssymb}
\usepackage{mathtools}
\usepackage{amsthm}
\usepackage{multirow}
\usepackage{subfigure}
\usepackage{bm}
\usepackage{xcolor}
\usepackage{booktabs}
\usepackage[capitalize,noabbrev]{cleveref}
\usepackage{tablefootnote}
\usepackage{tikz}
\usetikzlibrary{positioning, arrows.meta, shapes.geometric}

\newtheorem{theorem}{Theorem}

\newtheorem{lemma}{Lemma}

\newtheorem{example}{Example}

\newtheorem{definition}{Definition}
\theoremstyle{definition}
\newtheorem{assumption}{Assumption}

\theoremstyle{remark}

\newcommand{\eins}{\boldsymbol{1}}

\usepackage{newfloat}
\usepackage{listings}
\DeclareCaptionStyle{ruled}{labelfont=normalfont,labelsep=colon,strut=off} 
\title{PrAda-GAN: A Private Adaptive Generative Adversarial Network with Bayes Network Structure}
\author {
    Ke Jia\textsuperscript{\rm 1, \rm 2}\equalcontrib,
    Yuheng Ma\textsuperscript{\rm 1, \rm 2}\equalcontrib,
    Yang Li\textsuperscript{\rm 1, \rm 2},
    Feifei Wang \textsuperscript{\rm 1, \rm 2}\thanks{Corresponding author.}
}
\affiliations {
    \textsuperscript{\rm 1}Center for Applied Statistics, Renmin University of China\\
    \textsuperscript{\rm 2}School of Statistics, Renmin University of China\\
    {\{jiake1999, yma, yang.li, feifei.wang\}@ruc.edu.cn}
}

\usepackage{bibentry}

\begin{document}

\maketitle

\begin{abstract}
We revisit the problem of generating synthetic data under differential privacy.
To address the core limitations of marginal-based methods, we propose the Private Adaptive Generative Adversarial Network with Bayes Network Structure (\texttt{PrAda-GAN}), which integrates the strengths of both GAN-based and marginal-based approaches.
Our method adopts a sequential generator architecture to capture complex dependencies among variables, while adaptively regularizing the learned structure to promote sparsity in the underlying Bayes network.
Theoretically, we establish diminishing bounds on the parameter distance, variable selection error, and Wasserstein distance.
Our analysis shows that leveraging dependency sparsity leads to significant improvements in convergence rates.
Empirically, experiments on both synthetic and real-world datasets demonstrate that \texttt{PrAda-GAN} outperforms existing tabular data synthesis methods in terms of the privacy–utility trade-off.
\end{abstract}


\section{Introduction}
\label{sec:intro} 
Synthetic data is the new fossil fuel of modern AI, driving the success of multiple domains as models grow larger and demand unprecedented amounts of data for training \citep{wang2022self, gadre2023datacomp, lu2023machine}.
However, generative models are not immune to privacy risks—particularly membership inference attacks (MIAs), in which adversaries attempt to determine whether specific records are part of the training data. These vulnerabilities arise when generative models inadvertently memorize training samples, causing the synthetic outputs to closely resemble the originals or reveal exploitable statistical patterns \citep{sun2021adversarial, andrey2025tami}.

To address these risks, differential privacy \citep[DP,][]{dwork2006calibrating} is commonly applied during the training of generative models \citep{jordon2018pate, xie2018differentially}, ensuring that the model’s outputs remain indistinguishable regardless of whether any individual data point is included in the training set.
Numerous studies demonstrated DP generation of tabular data \citep{tao2021benchmarking, yang2024tabular, chen2025benchmarking}, which remains the most prevalent data type in data science \citep{hollmann2025accurate, zhang2025tabpfn}.
Among these, marginal-based methods, such as \texttt{PrivBayes} \citep{zhang2017privbayes} and \texttt{AIM} \citep{mckenna2022aim}, often achieve superior utility \citep{nist2019dpchallenge}.

However, marginal-based approaches have notable limitations.
First, they rely on low-dimensional structural assumptions for effective performance.
For instance, \texttt{PrivBayes} assumes a Bayes network structure, while \texttt{AIM} requires that the underlying distribution be well-approximated by a low-dimensional marginal structure in terms of workload error.
Although such assumptions are often reasonable, they are difficult to verify and to adaptively match to the true, unknown degree of low-dimensionality.
For example, if a few nodes have significantly more parent nodes than others, it becomes challenging to infer the overall network structure using the single hyperparameter in \texttt{PrivBayes} \citep{zhang2017privbayes}.
This issue is illustrated in Figure~\ref{fig:bayes-drawback}, where it may either underfit (by ignoring meaningful dependencies) or overfit (by including redundant ones).
Moreover, these marginal-based methods are primarily designed for categorical variables, requiring continuous variables to be discretized through binning.
In addition to the extra tuning cost involved in selecting appropriate bin sizes, this process also hinders the generation of heavy-tailed distributions.

\begin{figure}[htbp]
\vskip -0.1in
    \centering
    \subfigure[True network]{
    \begin{minipage}{0.3\linewidth}
        \centering
        \includegraphics[width=1\linewidth]{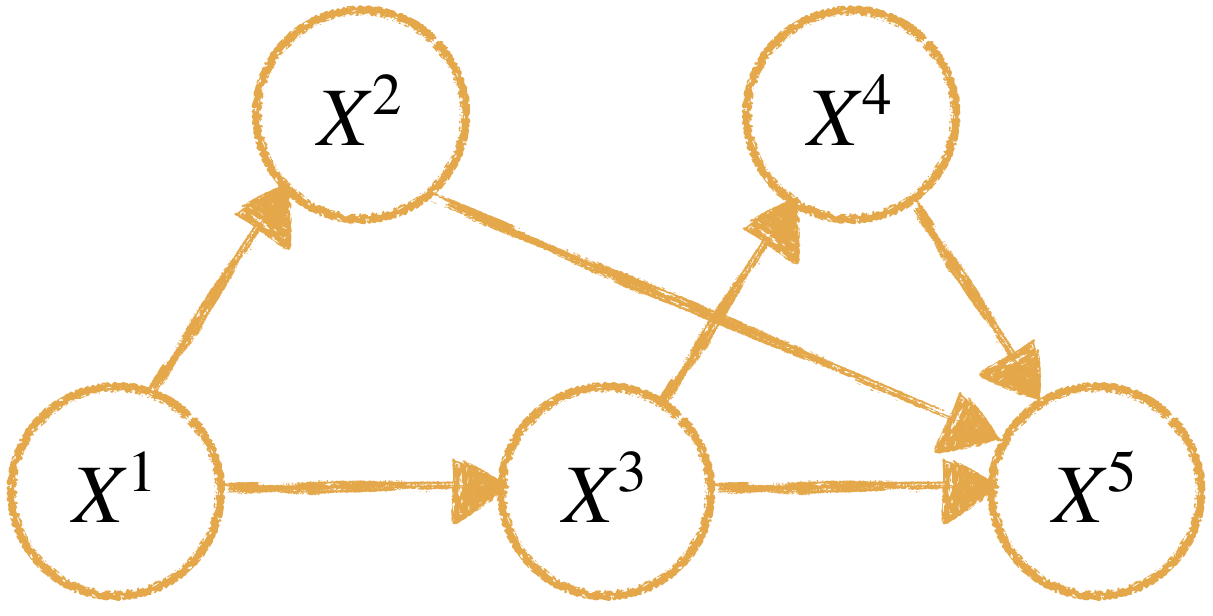}
        \label{fig:bayes1}
        \vskip -0.1in
    \end{minipage}
    }
    \subfigure[Underfitting]{
    \begin{minipage}{0.3\linewidth}
        \centering
        \includegraphics[width=1\linewidth]{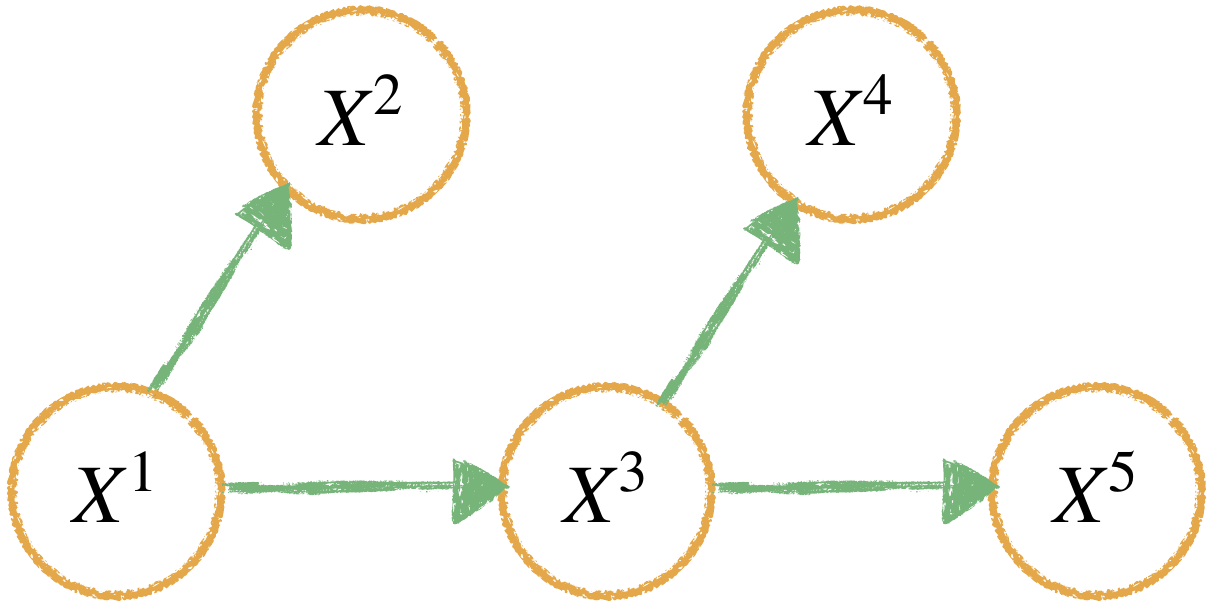}
        \label{fig:bayes2}
        \vskip -0.1in
    \end{minipage}
    }
    \subfigure[Overfitting]{
    \begin{minipage}{0.3\linewidth}
        \centering
        \includegraphics[width=1\linewidth]{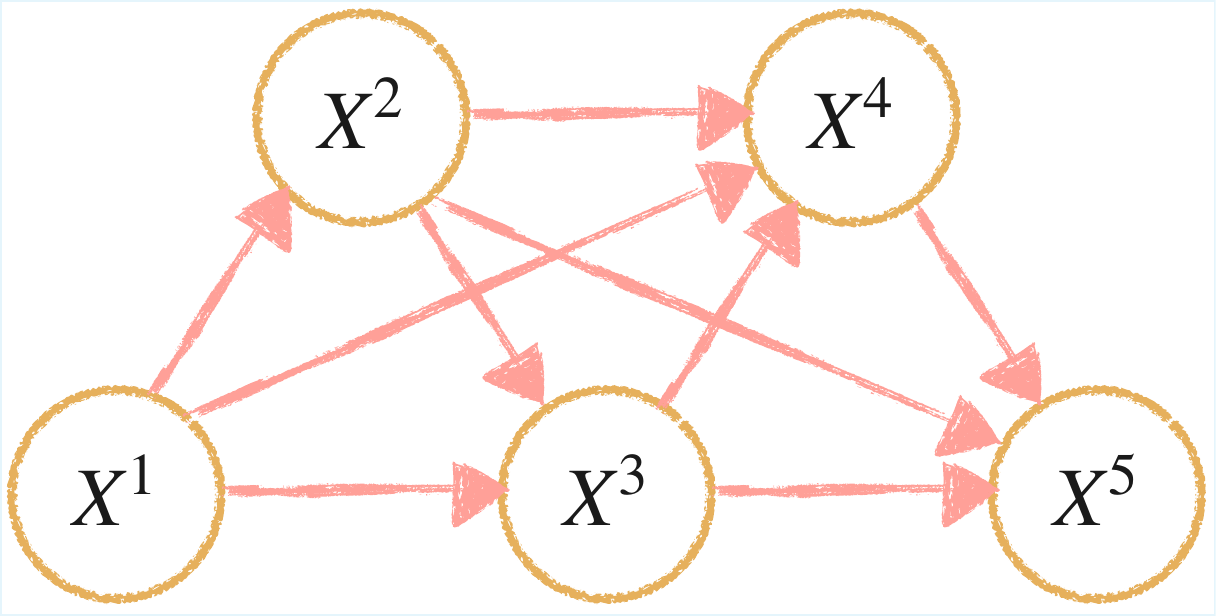}
        \label{fig:bayes3}
    \vskip -0.1in
    \end{minipage}
    }
    \vskip -0.15in
    \caption{Drawback illustration of marginal-based methods. }
    \label{fig:bayes-drawback}
    \vskip -0.1in
\end{figure}

In this paper, we address these challenges by proposing the Private Adaptive Generative Adversarial Network with Bayes Network Structure (\texttt{PrAda-GAN}), a novel approach for DP tabular data generation that integrates generative adversarial networks (GANs) with low-dimensional structural modeling. Our method employs a sequential generator architecture to capture complex dependencies among variables, while adaptively regularizing the learned structure to promote sparsity in the underlying Bayes network.
Compared to existing methods, \texttt{PrAda-GAN} offers two key advantages:
(1) it adapts to unknown low-dimensional structures without the need to tune sensitive hyperparameters; and
(2) it naturally supports unbounded continuous domains without requiring discretization.
Our contributions are: 
\begin{itemize}
    \item We revisit differentially private tabular data generation and identify key limitations of prior marginal-based approaches. To address these challenges, we propose \texttt{PrAda-GAN}, a novel method that combines the strengths of both marginal-based and GAN-based models. By introducing adaptive regularization, our approach implicitly recovers the underlying low-dimensional Bayes network structure during GAN training.

\item We provide a theoretical analysis of \texttt{PrAda-GAN}. First, we establish a bound on the distance between the trained generator and the optimal candidate set. Then, by analyzing the recovery of low-dimensional structures, we derive the first generalization bound for the Wasserstein distance between the generated and true data distributions. Notably, our results demonstrate that adaptive regularization leads to a significantly improved convergence rate.

\item We conduct extensive experiments on both synthetic and real-world datasets. Through a detailed analysis of parameter influence, we show that \texttt{PrAda-GAN} is robust to hyperparameter choices and supports our theoretical findings. We further benchmark \texttt{PrAda-GAN} against state-of-the-art baselines using measures of distributional similarity and downstream utility. The results demonstrate the empirical superiority of our approach.  
\end{itemize}

The remainder of the paper is organized as follows.
Section~\ref{sec:related work} reviews related work. 
Section~\ref{sec:method} introduces the proposed \texttt{PrAda-GAN} framework. Section~\ref{sec:theory} presents theoretical guarantees. Section~\ref{sec:exp} presents numerical results. Finally, Section~\ref{sec:con} concludes the paper.

\section{Related Work}
\label{sec:related work}

\paragraph{GAN Based Methods.}
Given the remarkable success of GANs \citep{goodfellow2014generative}, a growing body of works have explored their applications to differentially private data synthesis
\citep{jordon2018pate, liu2019ppgan, chen2020gs, long2021g, bie2023private,ma2023rdp}.
A common approach to privatizing GANs is to apply differentially private stochastic gradient descent (DPSGD) \citep{abadi2016deep} when updating the discriminator, a method known as DPGAN
\citep{xie2018differentially, zhang2018differentially, torkzadehmahani2019dp, zhao2024ctab}.
This technique is shown to be effective for generating private synthetic data across various domains.

\paragraph{Marginal-based Methods.}
Low-order marginals are widely adopted in tabular data synthesis due to their ability to capture essential low-dimensional structures while exhibiting low sensitivity under DP \citep{hu2024sok}. Marginal-based approaches typically select a set of marginals, inject calibrated noise, and reconstruct the joint distribution to generate synthetic data. A promising line of work employs Bayes networks to model conditional dependencies through a directed acyclic graph (DAG). Representative examples include \texttt{PrivBayes} \citep{zhang2017privbayes}, which learns the network structure from data and perturbs the conditional distributions, and AIM \citep{mckenna2022aim}, which enhances utility by tailoring the network structure to the available private marginals. Recent work established statistical foundations for marginal-based private synthesis \citep{li2023statistical}. 

An alternative line of research models data distributions using Markov random fields (MRFs)—undirected graphical models that capture symmetric relationships among variables. These methods estimate noisy low-order marginals and reconstruct the global distribution using inference techniques such as Gibbs sampling \citep{chen2015differentially, mckenna2019graphical, cai2021data}. By leveraging local Markov properties, MRF-based approaches are capable of representing complex dependencies while maintaining scalability and privacy.

\section{Proposed Method}
\label{sec:method} 
\subsection{Problem Definition}
\label{sec:private-synthesis}

We formalize the problem of synthesizing tabular data with DP. 
Suppose we have a random variable $X\in \mathcal{X}=\mathbb{R}^d$, whose distribution is $\mathrm{P}$. 
We have $n$ observations $\mathcal{D} = \{X_i\}_{i=1}^n$ from $\mathrm{P}$. 
Our target is to learn a generator $g_{{\bm\nu}}$, possibly parameterized by ${\bm\nu}$, such that for some easy-to-generate random variables $Z$, such as Gaussian or uniform random variables, the distribution of $g_{\mathbf{\nu}}(Z)$ is close to $\mathrm{P}$. 
The learned generator should preserve the privacy of training data, in the sense of differential privacy defined as follows.

\begin{definition}[Differential privacy \citep{dwork2006calibrating}]\label{def:dp}
A randomized algorithm $M: \mathcal{X}^n \rightarrow \mathcal{S}$ is $(\varepsilon, \delta)$-differentially private ($(\varepsilon, \delta)$-DP) if for every pair of adjacent data sets $\mathcal{D}, \mathcal{D}^{\prime} \in \mathcal{X}^n$ that differ by one datum and every $S \subseteq \mathcal{S}$, $\mathbb{P}(M(\mathbf{X}) \in S) \leq e^{\varepsilon} \cdot \mathbb{P}\left(M\left(\mathbf{X}^{\prime}\right) \in S\right)+\delta$,
where the probability measure $\mathbb{P}$ is induced by the randomness of $M$ only.    
\end{definition}
In this work, we consider the parameter space of ${\bm\nu}$ to be $\mathcal{S}$. Last, we define some notations. For any vector $x$, let $x^i$ denote the $i$-th element of $x$. 
We use the notation $a_n \lesssim b_n$ and $a_n \gtrsim b_n$ to denote that there exist positive constant $n_1$, $c$ and $c'$ such that $a_n \leq c b_n$ and $a_n \geq c' b_n$, for all $n \geq n_1$.
In addition, we denote $a_n\asymp b_n$ if $a_n\lesssim b_n$ and $b_n\lesssim a_n$.
Let $[n] = \{1,\ldots, n\}$. 
Let $a\vee b = \max (a,b)$ and $a\wedge b = \min (a,b)$. 
Besides, for any set $A\subset \mathbb{R}^d$, the diameter of $A$ is defined by $\mathrm{diam}(A):=\sup_{x,x'\in A}\|x-x'\|_2$. 
Let $\mathcal{W}(\mathrm{P}, \mathrm{Q})$ be the Wasserstein distance between distribution $\mathrm{P}$ and $\mathrm{Q}$. 

\subsection{Generators and Bayes Network}

We use an autoregressive approach to model the distribution. 
Specifically, let $\mathrm{P}[X]$ denote the joint distribution.
Then the joint distribution can be decomposed into 
\begin{align}\label{equ:autoregressive}
    \mathrm{P}[X] =
    \mathrm{P}[X^1] \cdot \prod_{j=2}^{d}  \mathrm{P}[X^j|X^1,\ldots, X^{j-1}].
\end{align}
To model \eqref{equ:autoregressive}, we use $d$ sub-generators to model the condition distributions of $X^j, j = 1, \ldots, d$, respectively. 
Specifically, let $Z^{j}, j=1,\ldots, d$ be some easy-to-general random variables, generated from $\mathrm{Q}$. 
Then, the $j$-th generator $g^j$ wants to model  
\begin{align}\label{equ:conditionalgenerator}
   g^{j}(X^1,\ldots, X^{j-1}, Z^j)  |X^{1:(j-1)} \sim  X^j  |X^{1:(j-1)}. 
\end{align}
Then the integrated generator is 
\begin{align*}
    g(Z) = \left(g^1(Z^1), g^2 (g^1(Z^1), Z^2), \ldots\right)^{\top}. 
\end{align*}

We adopt a similar assumption of Bayes network dependence as in \citep{zhang2017privbayes}, which significantly improves the efficiency of data generation.
A Bayes network over $\mathcal{X}$ provides a compact representation of the distribution by specifying conditional independencies among attributes in $\mathcal{X}$.
Specifically, a Bayes network is a DAG that represents each attribute in $\mathcal{X}$ as a node and uses directed edges to model the conditional dependencies between attributes.
The assumption is formally specified as follows. 

\begin{assumption}\label{asp:bayesnetwork}
    Assume that there exists a Bayes network $\mathcal{N} = \{(X^{j}, \Pi_j), j=1,\ldots, d\}$, such that : (i) $\Pi_j$ contains a subset of $[d]$, (ii) $X^{j}$ is only dependent on $\Pi_j$, and (iii) $j\notin \Pi_{i}$ for $i < j$. 
\end{assumption}

An example of a Bayes network is provided in Example  \ref{ex:bayesnetwork} in the appendix.
Under Assumption \ref{asp:bayesnetwork}, the autoregressive modeling \eqref{equ:autoregressive} can be further simplified as 
\begin{align}\label{equ:bayesnetautoregressive}
    \mathrm{P}[X] 
    = \mathrm{P}[X^1] \cdot \prod_{j=2}^{d}  \mathrm{P}[X^j| \Pi_j] =  \prod_{j=1}^{d}  \mathrm{P}[X^j| \Pi_j],
\end{align}
where we let $\Pi_{1} = \emptyset$. 
Under \eqref{equ:bayesnetautoregressive}, we can reduce the estimation in \eqref{equ:conditionalgenerator} into estimating the conditional relationship $X^j | \Pi_j$, which would reduce the intrinsic dimensionality.Denote the parameter of each generator $g^j$ by $\theta_j$.

One may argue that Assumption \eqref{asp:bayesnetwork} is too strong in practice. 
However, the marvelous performance of \citep{zhang2017privbayes} shows that the assumption is amenable since there exists a satisfiable $\mathcal{N}$ that captures most of the useful information in the conditional independence relationships at most of the times. 
Thus, an approximated $\mathcal{N}$ could be a nice surrogate to the true conditional relationship. 
Moreover, \citet{rojas2018invariant, zheng2018dags, wang2025dynamic} yield that the relationship can be well approximated. 
If there exist additional public datasets, whether in distribution or out of distribution, that share the same network structure, one can approximate the network and shift the order of the variable to let it satisfies Assumption \ref{asp:bayesnetwork}. 
The existence of such a similar public dataset is a common assumption in privacy-preserving machine learning \citep{yu2021large, ganesh2023public, ma2024decision, ma2024optimal, hod2025reallyneedpublicdata}.

\subsection{Private Generative Adversarial Network}

To generate high-dimensional, complex data (e.g., images), recent work has explored privatizing generative adversarial networks  \citep{goodfellow2014generative} to produce DP synthetic data via DPSGD \citep{abadi2016deep}, a line of research known as DPGAN \citep{xie2018differentially}. 
The common framework for GAN is the mini-max optimization problem 
\begin{align}\label{equ:objective-original-gan}
   \min_{g\in\mathcal{G}}\sup_{f\in \mathcal{F}} \bigg(\mathbb{E}_{X\sim \mathrm{P} }\left[ f(X)\right] -  \mathbb{E}_{Z\sim \mathrm{Q}} \left[ f(g(Z))\right]\bigg) ,
\end{align}
where $\mathcal{F}$ and $\mathcal{G}$ are the class of possible functions of discriminators and generators, and $Z$ is sampled from $\mathrm{Q}$. 
The solution of \eqref{equ:objective-original-gan} is obtained approximately through the minimization of an empirical objective
\begin{align}\label{equ:objective-original-gan-finite}
   \min_{g\in\mathcal{G}} \sup_{f\in \mathcal{F}} \left(\frac{1}{n}\sum_{i=1}^n f(X_i) -  \frac{1}{n_g}\sum_{i=1}^{n_g} f(g(Z_i))\right),
\end{align}
where an observation of $n$ real samples $\{X_{i}\}_{i=1}^n$ and $n_g$ easy-to-sample samples $\{Z_{i}\}_{i=1}^{n_g}$ are available.
Denote the parameter of $g$ and $f$ by $\bm \theta$ and $\bm \nu$, respectively.
We denote the objective
\begin{align}\label{equ:def-of-objective}
    \Delta({\bm\theta}, {\bm \nu}, D) = \frac{1}{n}\sum_{i=1}^n f_{\bm \nu}(X_i) -  \frac{1}{n_g}\sum_{i=1}^{n_g} f_{\bm \nu}\left(g_{\bm \theta}(Z_i)\right). 
\end{align}

Finding the exact solution of \eqref{equ:objective-original-gan-finite} is usually infeasible if the function class $\mathcal{F}$ and $\mathcal{G}$ are complex enough, as the optimization is usually non-convex. 
Thus, the optimization of \eqref{equ:objective-original-gan-finite} is conducted via iteratively minimizing w.r.t. $g$ and maximizing w.r.t. $f$ over the objective function. 
The update of $g_{\bm \theta}$ is done by stochastic gradient descent ${\bm \theta}^{t+1} = {\bm \theta}^{t} - \eta_{\theta} \nabla_{\bm \theta} \Delta({\bm \theta}^{t}, {\bm\nu}^{t}, D)$.
The maximization over $f$ should be conducted under the constraint of DP. 
Specifically, the update of $f_{\bm \nu}$ is conducted with DPSGD \citep{abadi2016deep}, denoted as ${\bm \nu}^{t+1}={\bm \nu}^{t} + \eta_{\nu}$ \texttt{PrivGrad}$(\nabla_{\bm \nu} \Delta({\bm \theta}^{t}, {\bm\nu}^{t}, D), \sigma)$, where $\sigma$ is the privacy noise level. 
Note that the minimization over $g$ is independent of the real data and is thus free of privacy concerns. 
The optimization of $\bm \theta$ and $\bm \nu$ conducted iteratively, meaning that one should update ${\bm \theta}^{t+1}$, use ${\bm \theta}^{t+1}$ to update ${\bm \nu}^{t+1}$, and continue this process. 
In practice, however, due to the performance drop brought by DPSGD during optimizing $f_{\bm \nu}$, the update of discriminator $f_{\bm \nu}$ should proceed multiple times before one update of the generator $g_{\bm \theta}$,  as argued by \citet{bie2022private}. 
Thus, we introduce an additional parameter $t_g$ to account for this relative number of iterations.

\subsection{Adaptive Feature Selection}

The determination of $\Pi_j$ is tricky. 
\citet{zhang2017privbayes} utilized a private version of a surrogate of mutual information to determine the network. 
This approach is, however, highly restrictive to the choice of the degree of the networks, as illustrated in Section \ref{sec:intro}. 
We propose a data-driven feature selection rule that leverages a sparsity-inducing penalty. 
Specifically, we consider 
$\bm{\theta}_j = (\bm{\xi}_j, {\bf W}_j)$, where ${\bf W}_j$ represents the linear feature map from $j$ dimensional space onto an arbitrary dimensional, say $L_j$, space, and $\bm{\nu}_j$ represents the parameters governing the map from this feature to the output. 
This is also known as single index models in the statistics community \citep{xia2008multiple}. 
This leads to 
\begin{align}\label{equ:singleindex}
    g^{j}  (X^1,\ldots, X^{j-1}, Z^j) = \tilde{g}^{j}_{\bm{\nu}_j } \left( {\bm W}_j^{\top } \left[ X^{1:{j-1}}, Z_j\right] \right). 
\end{align}
Here, ${\bm W}_j$ is a $\mathbb{R}^{j \times L_j}$ matrix. 
Model \eqref{equ:singleindex} includes a large class of functions, including neural networks.
To induce sparsity, we penalize the sum of $L_2$ norm of weights associated to each feature following \citep{feng2017sparse, dinh2020consistent, wang2024penalized}, formally 
\begin{align*}
    L({\bm W}_j) = \sum_{k = 1}^{j-1} \|{\bm W}_j^{k, :}\|_2.
\end{align*}
We refer to this penalty as a group lasso penalty, in the sense that it penalizes the entire set of weights corresponding to a feature \citep{yuan2006model, friedman2010note, simon2013sparse}. Unlike a standard $L_2$ penalty, it has a composite structure: an $L_2$ norm is applied within each group (a row of ${\bm W}_j$), followed by an $L_1$-type aggregation across groups (the collection of row norms). This structure enables all weights associated with a single feature to be shrunk to zero simultaneously, thereby inducing sparsity at the feature level.
Thus, we optimize a penalized version of objective function \eqref{equ:def-of-objective} 
\begin{align}\label{equ:objective-ours-finite}
 & \tilde{\Delta}({\bm\theta}, {\bm \nu}, D) \\
 = &   \frac{1}{n}\sum_{i=1}^n f(X_i) -  \frac{1}{n_g}\sum_{i=1}^{n_g} f(g(Z_i))+ \sum_{j = 1}^d \lambda_{j} L({\bm W}_j).\nonumber
\end{align}
Here, $\lambda_1,\ldots, \lambda_d$ are pre-determined tuning parameters.
The overall optimization process is illustrated in Algorithm \ref{alg:privdag-gan}.

\begin{algorithm}[h]
   	\caption{\texttt{PrAda-GAN}}
   	\label{alg:privdag-gan}
    {\bfseries Input: }{ Private data ${D} = \{X_i\}_{i=1}^n$.   }\\
     {\bfseries Parameters: }{ Learning rate $\eta_{\bm \theta}, \eta_{\bm \nu}$, iteration number $T$, relative number of iterations $t_g$,  
     threshold $\tau$, regularizations $\lambda_1,\ldots, \lambda_d$, noise level $\sigma$.  }\\
     {\bfseries Initialization: }{ Initial parameters $\bm \theta^1$, $\bm \nu^1$.   }\\
     \For{$t$ in $[T]$}{
     \texttt{\# Update $\nu$ every step.  }\\
     ${\bm \nu}^{t+1}={\bm \nu}^{t} + \eta_{\nu}\cdot $ \texttt{PrivGrad}$(\nabla_{\bm \nu} \tilde{\Delta}({\bm \theta}^{t}, {\bm\nu}^{t}, D), \sigma)$  \\
    \texttt{\# Update $\theta$ every $t_g$ steps.  } 
     \\
     ${\bm \theta}^{t+1} = {\bm \theta}^t - \eta_{\theta} \cdot \eins \left(t\mod t_g \equiv 0\right) \cdot \nabla_{\bm \theta} \tilde{\Delta}({\bm \theta}^{t}, {\bm\nu}^{t+1}, D)$. 
    }
     
    {\bfseries Output: }{Generator $g_{\bm \theta^{T+1}}$ and discriminator $f_{\bm \nu^{T+1}}$.}
   \end{algorithm} 

\vskip -0.8in

\section{Theoretical Guarantees}
\label{sec:theory}

\subsection{Assumptions}

Denote the assumed function classes for generators and discriminators as $\mathcal{G}$ and $\mathcal{F}$.

\begin{assumption}[\textbf{Effective Private Optimization}]\label{asp:private-optimization}
The output of Algorithm \ref{alg:privdag-gan} achieves an optimization result of 
\begin{align}
    \tilde{\Delta}({\bm \theta}^{T+1}, {\bm \nu}^{T+1}, D) \leq \inf_{{\bm \theta}\in \mathcal{G}} \sup_{{\bm \nu}\in \mathcal{F}}\tilde{\Delta}({\bm \theta}, {\bm \nu}, D) + \mathcal{E}  
\end{align}
for some $\mathcal{E}>0$. 
Suppose $\mathcal{E}$ is diminishing as $n\to \infty$. 
\end{assumption}

Assumption~\ref{asp:private-optimization} is present to control the performance drop brought by the inexact optimization of \eqref{equ:objective-ours-finite} using DPSGD. 
There are mainly two error sources in $\mathcal{E}$, guaranteed respectively by 
\begin{align}\label{equ:private-optimizaiton-error}
    \tilde{\Delta}({\bm \theta}^{T+1}, {\bm \nu}^{T+1}, D) \leq \tilde{\Delta}({\bm \theta}_{np}, {\bm \nu}_{np}, D) + \mathcal{E}_{priv}
\end{align} 
and 
\begin{align}\label{equ:optimizaiton-error}
     \tilde{\Delta}({\bm \theta}_{np}, {\bm \nu}_{np}, D) \leq \inf_{\bm \theta} \sup_{\bm \nu}\tilde{\Delta}({\bm \theta}, {\bm \nu}, D) +  \mathcal{E}_{opt}, 
\end{align} 
and thus $\mathcal{E} = \mathcal{E}_{priv} +  \mathcal{E}_{opt}$ satisfies \eqref{asp:private-optimization}. 
Specifically, the presence of \eqref{equ:private-optimizaiton-error} is due to the weakened optimization performance of DPSGD compared to SGD. 
Namely, each ${\bm \nu}^{t}$ achieves (in expectation) a larger $\tilde{\Delta}$ than ${\bm \nu}_{np}^{t}$, and thus has a smaller discriminative capability. 
Thus, perceiving information from a weaker discriminator, the generator $g_{\bm \theta}$ is, in general, less effective.
Other approaches to mitigate this issue include a more powerful remedy of DPSGD \citep{bu2023automatic, liu2025towards}, longer discriminator training \citep{bie2023private}, or additional information assistance \citep{yu2021large, yu2022differentially}. 
In general, $\mathcal{E}_{priv}$ should be smaller when $\varepsilon$ is large. 
For \eqref{equ:optimizaiton-error}, error $\mathcal{E}_{opt}$ represents the gap between optimization SGD and the exact solution, and is often referred to as optimization error. 
Further verification of the assumption and establishment of bounds on $\mathcal{E}$ should refer to the literature in training dynamics of GANs \citep{biau2020some, xu2020understanding, huang2022error} and differentially private stochastic optimization \citep{bassily2021differentially, su2024faster}.

\begin{assumption}[\textbf{Finite Moment}]\label{asp:finite-moment}
 Let $X \sim \mathrm{P}$. 
 For some $c_m>0$, $X$ satisfies the first moment tail condition $\mathbb{E}\left[\|X\| \eins\left(\|X\|>\log t\right)\right] = O\left(t^{-(\log t)^{c_1} /d }\right)$, for any $t \geq 1$.
\end{assumption}



\begin{assumption}[\textbf{Analytic Generator}]\label{asp:analytic}
   The function $g_{\bm \theta}$ is analytic with respect to $\bm \theta$. 
\end{assumption}

Assumption \ref{asp:finite-moment} requires $\mathrm{P}$ is concentrated, and is commonly satisfied, for instance by sub-Gaussian variables. 
Assumption \ref{asp:analytic} is easily satisfied with several choices of activation functions, including the classic ones such as the linear function, tanh function, and sigmoid function, as well as the newly developed ReLU-type activation functions such as GeLU, ELU, and PELU.


\subsection{Parameter Estimation and Feature Selection}

In this section, we demonstrate that under the assumptions outlined in the previous section, the generator obtained from Algorithm \ref{alg:privdag-gan} exhibits superiority in both parameter estimation and feature selection. We first define the relevant criteria for evaluation.
For parameter estimation, we consider the quantity of $d\left({\boldsymbol{\theta}}^{T+1}, \Theta\right)$ defined as follows.
Let
\begin{align}
\Theta = \left\{{\bm \theta} \in \mathcal{G} \mid g_{\bm \theta} \in {\arg \min}_{\bm \theta } \mathcal{W} \left( \mathrm{P}_{g_{\bm \theta}(Z)}, \mathrm{P}_X\right)\right\}, 
\end{align}
denote the set of feasible parameters that achieve the minimum Wasserstein distance within $\mathcal{G}$. 
Then, 
\begin{align*}
    d\left({\boldsymbol{\theta}}^{T+1}, \Theta\right) =  { \min}_{{\bm \theta}\in \Theta} d({\bm \theta}^{T+1}, {\bm \theta}), 
\end{align*}
represents the closest distance between ${\bm \theta}^{T+1}$ and any element in $\Theta$. 
For feature selection, we evaluate the sum of norms associated with unimportant features (i.e., those outside $\Pi_j$): $\sum_{k \in \Pi_j^c }\left\| \left({\bm W}_j^{T+1}\right)^{k}\right\|_2$. 
Both quantities are expected to be small. 
Define $\mathcal{F}_{Lip, 1}$ be the 1-Lipschitz class on $\mathbb{R}^d$. 
Given these definitions, we have the following theorem.

\begin{theorem}\label{thm:feature-selection}
    Suppose that Assumptions \ref{asp:bayesnetwork}, \ref{asp:private-optimization}, \ref{asp:finite-moment}, and \ref{asp:analytic} hold. 
    Suppose that $\lambda_j=d^{-1}\psi^{1/2}_{n,d} , j = 1,\ldots,d$,
    where we define 
\begin{align}
\nonumber
\psi_{n, d} = &
(\sqrt{d} + \log n) n^{- \frac{1}{d}}
+    
\mathcal{E}\\
+& \sup_{f \in \mathcal{F}_{Lip, 1}}  \inf_{f' \in \mathcal{F}} \left\|f - f'\right\|_{\infty}. 
\label{equ:def-of-psi}
\end{align}
Then, the fitted parameters $  {\bm \theta}^{T+1}$ and ${\bm W }^{T} \in  {\bm \theta}^{T+1}$ from Algorithm \ref{alg:privdag-gan} (without comment) satisfy
\begin{align}\label{equ:thm-bound-of-feature-selection}
    \mathbb{E}[  \sum_{k \in \Pi_j^c }\| ({\bm W}_j^{T+1})^{k}\|_2] \lesssim  d \cdot \psi_{n,d}^{\frac{1}{2(a-1)}}, \;\;\;\; j\in [d],
\end{align}
as well as 
\begin{align}\label{equ:thm-bound-of-estimation}
    \mathbb{E}\left[d\left({\boldsymbol{\theta}}^{T+1}, \Theta\right) \right] \lesssim \psi_{n,d}^{\frac{1}{2(a-1)}}.
\end{align}
Here, $a>2$ is a positive constant.
The expectation $\mathbb{E}$ is taken w.r.t. training samples $\left\{X_i\right\}_{i=1}^{n}$ and $\left\{ Z_i \right \}_{ i=1}^{ n}$.
\end{theorem}

The theorem are interpreted as follows. 
Both bounds in \eqref{equ:thm-bound-of-estimation} and \eqref{equ:thm-bound-of-feature-selection} depend on $\psi_{n,d}$, which is expected to diminish. 
This quantity consists of three components: the estimation error, 
the approximation error, and the private optimization error. 
The estimation error arises due to the discrepancy between the empirical and population objective function, which decreases with larger $n$ and smaller $d$. 
The approximation error measures how well the discriminator function class $\mathcal{F}$ approximates the 1-Lipschitz class, which should be smaller for $\mathcal{F}$ that is more expressive. 
The private optimization error reflects the gap between privately optimized and globally optimal objective functions. 
See also Assumption \ref{asp:private-optimization} and comments below. 
Using developed tools \citep{abadi2016deep, bu2023automatic, liu2025towards}, one should expect this term to diminish with larger $n$, larger $\varepsilon$, and smaller $d$, yet increase when the function classes $\mathcal{G}$ and $\mathcal{F}$ are most complex, i.e., harder to optimize. 
Thus, when $n$, $\varepsilon$, and $d$ are fixed, the choice of function classes and optimizers should balance the approximation and optimization errors.

Combining the three terms together, Theorem \ref{thm:feature-selection} suggests that when $n\to \infty$, and $d,\varepsilon$ both remain moderately small, $\psi_{n,d}$ would diminish. 
Consequently, the obtained $\bm \theta^{T+1}$ not only approaches the optimal parameter set $\Theta$, but also converges to a well-behaved class with zero dependence on redundant variables.
Theorem \ref{thm:feature-selection} applies to general function classes $\mathcal{G}$ and $\mathcal{F}$. 
Specific convergence rates for neural networks can be derived as in \citet{dinh2020consistent, wang2024penalized}. 
Also, for the rate $n^{-\frac{1}{d}}$ to diminish, we require $d = o(\log n)$. 
Here, we implicitly assume $d=O(\log^{c_1} n)$ for $0<c_d<1$.

\subsection{Convergence Leveraging Sparsity}

One should note that the closeness of $d\left({\boldsymbol{\theta}}^{T+1}, \Theta\right)$ does not imply any convergence result for $\mathcal{W} \left( \mathrm{P}_{g_{\bm \theta}(Z)}, \mathrm{P}_X\right)$. 
Moreover, the convergence in \eqref{equ:thm-bound-of-estimation} does not reflect the improvement afforded by sparsity (i.e., Assumption \ref{asp:bayesnetwork}).
In this section, we propose a variant of Algorithm \ref{alg:privdag-gan} by incorporating weight thresholding.
Specifically, we exclude variables with small parameter norms by setting their associated weights to zero.
See the commented code in lines 9–18 of Algorithm \ref{alg:privdag-gan-two-step} in the appendix.
By leveraging this approach, we establish a convergence rate that depends on the underlying sparsity structure $|\Pi_j|, j\in [d]$, rather than the full dimensionality.

For the theoretical analysis, we introduce an additional assumption regarding the second-phase optimization, analogous to Assumption \ref{asp:private-optimization}. 
Let $\overline{\mathcal{G}}$ denote the class of functions where weights $\bm W$ are set to zero if $\left\|(\bm W^{T+1}_j)^{k,:}\right\|_2 \leq \tau$. 
Note that this class is data-dependent.

\begin{assumption}[\textbf{Constraint Effectiveness}]\label{asp:private-optimization-second-stage}
The commented output of Algorithm \ref{alg:privdag-gan}, i.e. $\overline{\bm \nu}^{T+1}$ and $\overline{\bm \theta}^{T+1}$ achieves an optimization result of 
\begin{align}
    {\Delta}(\overline{\bm \theta}^{T+1}, \overline{\bm \nu}^{T+1}, D) \leq \inf_{{\bm \theta}\in \overline{\mathcal{G}}} \sup_{{\bm \nu}\in \mathcal{F}}
    {\Delta}({\bm \theta}, {\bm \nu}, D) + \overline{\mathcal{E}}  
\end{align}
for some $\overline{\mathcal{E}} > 0$. 
Suppose $\overline{\mathcal{E}}$ is diminishing as $n\to \infty$. 
\end{assumption}

This assumption ensures satisfactory optimization performance over the post-feature-selection function class $\overline{\mathcal{G}}$ and is no more restrictive than Assumption \ref{asp:private-optimization}.
In practice, however, one may prefer to run Algorithm \ref{alg:privdag-gan} without explicitly applying the thresholding step (as commented in the code), provided the influence of unimportant variables remains negligible.
However, if the true dependence structure is highly sparse, the improvements brought by this feature selection could be significant, e.g., \citep{ma2024better, kent2024rate}. 
Notably, this guarantee holds even if we simply clip all weights $\left\|(\bm W^{T+1}_j )^{k, :} \right \|_2\leq \tau$ to zero without re-optimization.
Let $g_{\bm \theta^{T+1}_{clip}}$ denote the resulting clipped generator .
Then $g_{\bm \theta^{T+1}_{clip}}\in \overline{\mathcal{G}}$ and under sufficiently small penalization parameters $\lambda_j$, $g_{\bm \theta^{T+1}_{clip}}$ satisfies Assumption \ref{asp:private-optimization-second-stage}.

The following theorem establishes a distance convergence guarantee for the commented output $\overline{\bm \theta}^{T+1}$ of Algorithm \ref{alg:privdag-gan}.
Define the ancestor set $J_{\infty}^j$ of feature $j$ by $J_0^j = \{j\}$, $J_{k+1}^j = \bigcup_{j \in J_{k}^j} \Pi_j $, and $J_{\infty}^j = \lim_{k\to \infty} J_{k}^j$.

\begin{theorem}\label{thm:estimation}
 Suppose that the assumptions in Theorem \ref{thm:feature-selection} and Assumption \ref{asp:private-optimization-second-stage} hold. 
    Suppose we set $\lambda_j=d^{-1}\psi^{1/2}_{n,s} , j = 1,\ldots,d$,
    where we define 
\begin{align}
\nonumber
\psi_{n, s} = &
(\sqrt{d} + \log n)\; n^{- \frac{1}{s}}
+ \overline{\mathcal{E}}
\\
+&
\label{equ:def-of-psi-phase-2}
\sup_{f \in \mathcal{F}_{Lip, 1}}  \inf_{f' \in \mathcal{F}} \left\|f - f'\right\|_{\infty}. 
\end{align}
Here we have $s = \max_{j\in[d]}J_{\infty}^j$. 
Then, with probability $1$, there exists a suitable choice of $\tau$ such that the fitted parameters $  \overline{\bm \theta}^{T+1}$ and $ \overline{\bm W }^{T+1} \in  \overline{\bm \theta}^{T+1}$ from Algorithm \ref{alg:privdag-gan} (with comment) satisfy 
\begin{align}
   \mathbb{E}\left[ \mathcal{W} \left( \mathrm{P}_{g_{\overline{\bm \theta}^{T+1}}(Z)}, \mathrm{P}_X\right)\right] \lesssim \psi_{n,s}. 
\end{align}
The expectation $\mathbb{E}$ is taken w.r.t.  $\left\{X_i\right\}_{i=1}^{n}$ and $\left\{ Z_i \right \}_{ i=1}^{ n}$.
\end{theorem}

The theorem illustrates the generative performance of the obtained parameter $\overline{\bm \theta}^{T+1}$. 
Note that $\psi_{n, s}$ in \eqref{equ:def-of-psi-phase-2} consists of three parts analogous to those in \eqref{equ:def-of-psi}, but the estimation error is significantly improved from $-1/d$ to $-1/s$.
In many cases, features can be divided into several independent groups, allowing $s$ to be much smaller than $d$, illustrating the improvement brought by the Bayes network structure.

The relationship between $\overline{\mathcal{E}}$ and $\mathcal{E}$ is obvious - intuitively, the expectation of $\overline{\mathcal{E}}$ is always no larger than $\mathcal{E}$. 
This is because $\mathcal{E}$ and $\overline{\mathcal{E}}$ only account for optimization error, and the constraint optimization problem $\inf_{{\bm \theta}\in \overline{\mathcal{G}}} \sup_{{\bm \nu}\in \mathcal{F}}
    {\Delta}({\bm \theta}, {\bm \nu}, D)$ is easier than $\inf_{{\bm \theta}\in {\mathcal{G}}} \sup_{{\bm \nu}\in \mathcal{F}}
    {\Delta}({\bm \theta}, {\bm \nu}, D)$. 
    However, explicitly depicting such improvement again yields an involved analysis of the dynamics of GAN. 

\section{Experiment}
\label{sec:exp}

In the experiments, we perform both synthetic and real-world experiments in Section \ref{nonlinear_exp} and \ref{real_data}. 
More details on datasets and settings can be found in Appendix \ref{exp_settings}.

\paragraph{Privacy Analysis}
We adopt DP\text{-}SGD as optimizer and evaluate its privacy guarantees through the R\'enyi differential privacy (RDP) accountant \citep{mironov2017renyi} implemented in Opacus. For each experiment, we specify a target privacy budget~$\varepsilon$ and select the noise multiplier~$\sigma$ and the total number of training rounds~$T$ so that the privacy loss computed by the RDP accountant matches the desired~$\varepsilon$ (with~$\delta$ fixed as described in Appendix~\ref{exp_settings}). The RDP framework provides tight privacy tracking for iterative gradient updates, which makes it well suited for DP\text{-}SGD. Through this procedure, all reported results are $(\varepsilon, \delta)$\nobreakdash-DP.

\paragraph{Evaluation Metrics} 
(1) \textbf{Distribution Similarity}: We evaluate the generator's ability to capture distribution similarity by computing the total variation distance (TVD), computed via 2-way marginals, and the Wasserstein distance (WD). 
See more details in Appendix \ref{sec:evaluation}. 
(2) 
\textbf{Machine Learning Efficacy}:
 For real datasets, we assess the quality of synthetic data through downstream machine learning performance. We employ five robust machine learning models known for their strong generalization: MLP \citep{rumelhart1986learning}, CatBoost \citep{prokhorenkova2018catboost}, XGBoost \citep{chen2016xgboost}, Random Forest \citep{breiman2001random}, and SVM \citep{cortes1995support}. 
 Performance is evaluated using the average $R^2$ score and root mean square error (RMSE) on held-out test data, with respect to a model trained on synthetic data.

\subsection{Synthetic Experiments}
\label{nonlinear_exp}

\paragraph{Experiment Setup} We generate synthetic data according to the structural equation model $X_j = f_j(\Pi_j) + z_j$, $j \in [d]$, where $z_j \sim \mathcal{N}(0,1)$, are independent noise terms and recall that $\Pi_j$ denotes the parent set of $X_j$ in the underlying Bayes network. 
Each function $f_j$ is modeled using a multiple index model, represented as a sum of nonlinear transformations applied to one-dimensional linear projections of the parent variables $\Pi_j$.  
This flexible form ensures expressiveness while maintaining identifiability under mild regularity conditions. The ground truth Bayes networks are generated from the Erdős-Rényi graph model \citep{zheng2018dags}, with a fixed number of nodes $d=10$. For each generated network, we simulate $n = 2000$ training samples. 
Parallel sets of experiments using linear functions $f_j$ and alternative evaluation metrics are presented in Appendix~\ref{exp_results}.

\paragraph{Parameter Analysis of Learning Rate} 
We investigate the impact of the discriminator learning rate $d_{\text{lr}}$ and the generator learning rate $g_{\text{lr}}$. 
We use the nonlinear function $f_j(\Pi_j ) = \tanh(w_{1j}^\top \Pi_j) + \cos(w_{2j}^\top \Pi_j) + \sin(w_{3j}^\top \Pi_j)$, where $w_{kj}^\top \Pi_j = \sum_{j'\in \Pi_j} w_{k,j,j'} X^{j'}$.
Each weight $w_{k,j,j'}$ is randomly initialized by sampling from a uniform distribution over $[0.5, 2.0]$, and independently negated with probability $0.5$.
We vary $d_{\text{lr}}$ while fixing $g_{\text{lr}}$, and vice versa. 
Each parameter setting is repeated 20 times.
Performance is evaluated under different privacy noise levels $\sigma \in \{1, 3, 5, 7\}$. 
The results are shown in Figure~\ref{fig:g_lr_d_lr}. For each $\sigma$, as $d_{\text{lr}}$ increases, both WD and TVD initially decrease, reaching a minimum at a certain value of $d_{\text{lr}}$, and then begin to increase. A similar trend is observed when varying $g_{\text{lr}}$.
These observations indicate the existence of an optimal learning rate that primarily governs the optimization error $\mathcal{E}$.

\begin{figure}[!h]
\vskip -0.1in
    \centering
    \begin{minipage}{\columnwidth}
        \centering
        \includegraphics[width=0.65\linewidth]{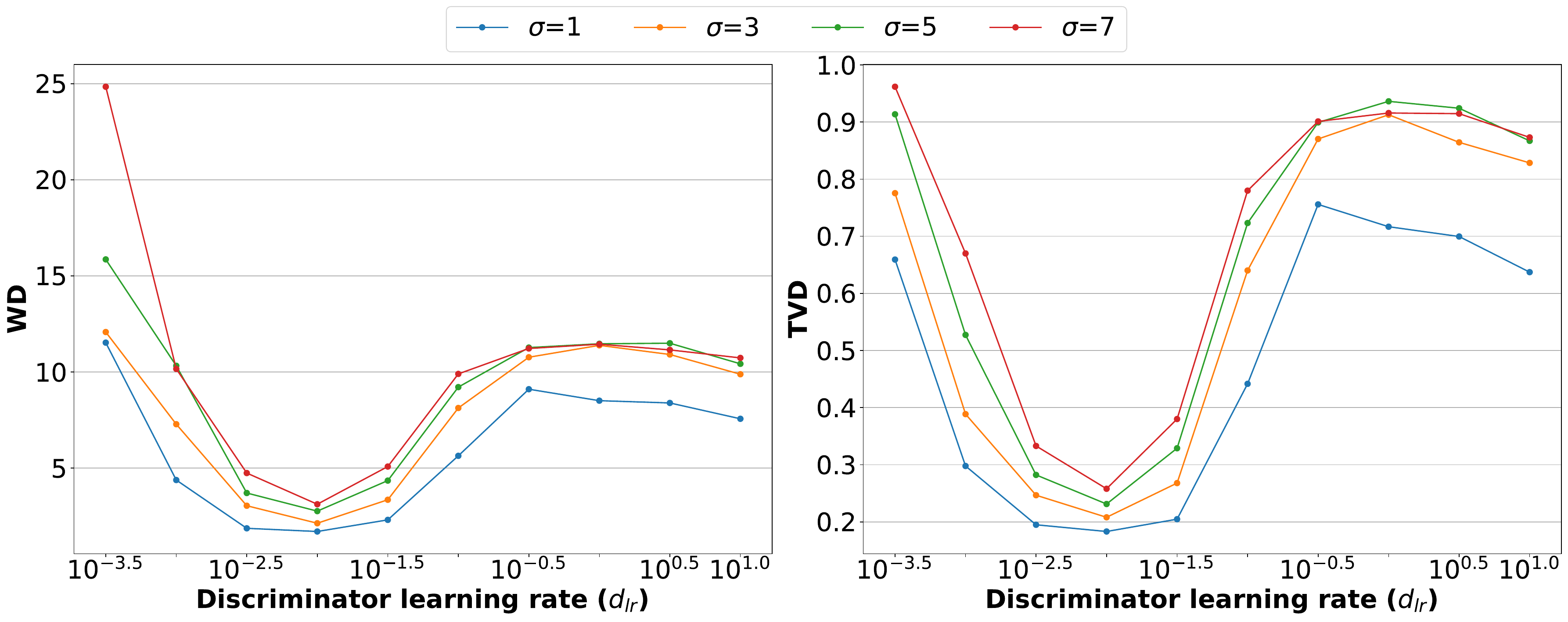}
        \label{fig:d_lr}
    \end{minipage}
    \\  
    \begin{minipage}{\columnwidth}
        \centering
        \includegraphics[width=0.65\linewidth]{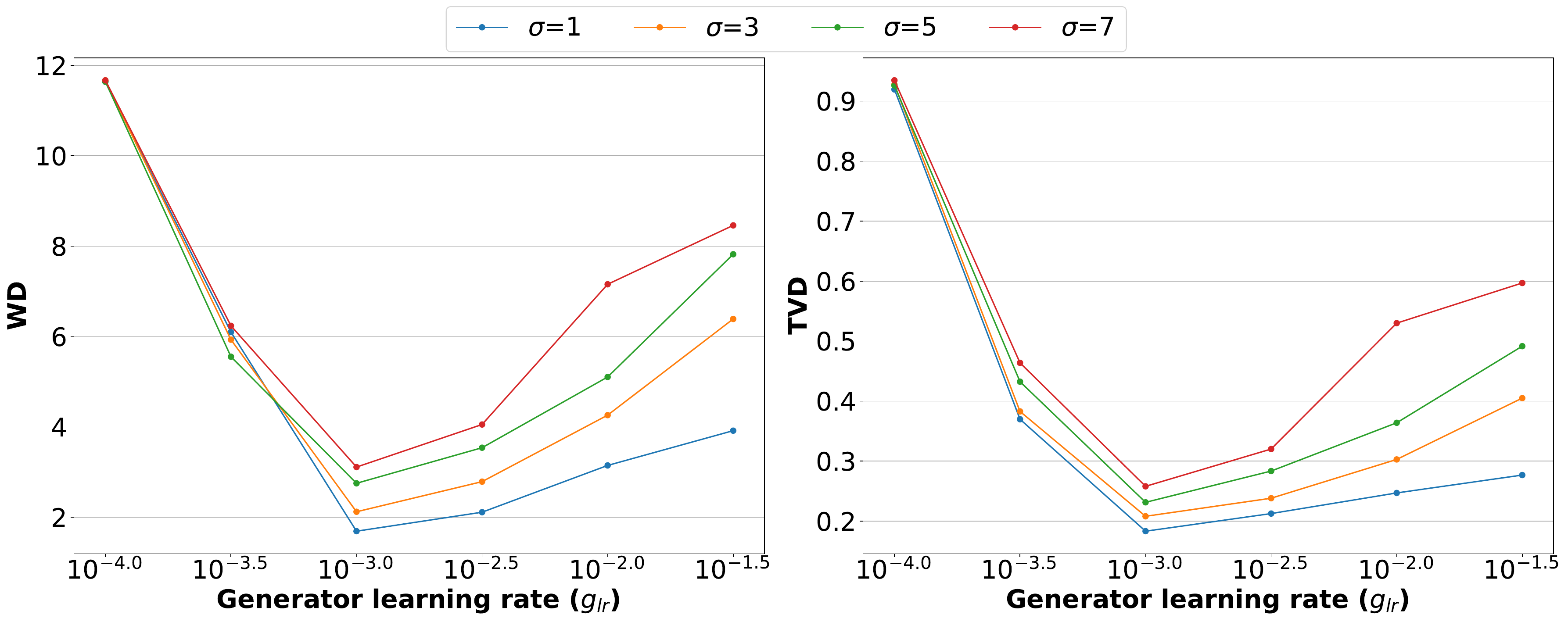}
        \label{fig:g_lr}
    \end{minipage}
    \vskip -0.1in
    \caption{Average WD and TVD. Top: varying discriminator learning rates \( d_{\text{lr}} \) across $10^{h}$ \( g_{\text{lr}} = 10^{-2} \). Bottom: varying generator learning rates \( g_{\text{lr}} \) across $10^{h}$ with \( d_{\text{lr}} = 10^{-1} \).
}
    \label{fig:g_lr_d_lr}
    \vskip -0.1in
\end{figure}


\begin{table*}[!t]
\centering
\tiny
\caption{
Real data comparison for distribution similarity.
Best results are in \textbf{bold}; second-best are \underline{underlined}. 
To ensure statistical significance, we adopt the Wilcoxon signed-rank test  \citep{wilcoxon1992individual} with a significance level of 0.05 to check if the result is significantly better.
The best results that hold significance over the others have a *.
}
\setlength{\tabcolsep}{2pt}
\vskip -0.05in
\label{tab:similarity_results}
\resizebox{0.8\textwidth}{!}{%
\begin{tabular}{lcccccc cccccc cccccc}
\toprule
\textbf{Dataset} & \multicolumn{6}{c}{\textbf{\textsc{California} }} & \multicolumn{6}{c}{\textsc{House-16H}} & \multicolumn{6}{c}{\textsc{Cpu-act}} \\
\cmidrule(lr){2-7} \cmidrule(lr){8-13} \cmidrule(lr){14-19}
$\varepsilon$ & \multicolumn{2}{c}{0.2} & \multicolumn{2}{c}{1} & \multicolumn{2}{c}{5} 
& \multicolumn{2}{c}{0.2} & \multicolumn{2}{c}{1} & \multicolumn{2}{c}{5} 
& \multicolumn{2}{c}{0.2} & \multicolumn{2}{c}{1} & \multicolumn{2}{c}{5} \\
Metric & WD & TVD & WD & TVD & WD & TVD 
& WD & TVD & WD & TVD & WD & TVD 
& WD & TVD & WD & TVD & WD & TVD \\
\midrule
\texttt{PrAda-GAN} & \textbf{2.034}* & \textbf{0.354}* & \textbf{1.533}* & \textbf{0.287}* & \textbf{1.235}* & \textbf{0.257}* & \textbf{4.498}* &  \underline{0.355} & \textbf{3.149}* & \textbf{0.260} & \textbf{3.121}* & \underline{0.220} & \textbf{7.628}* & \textbf{0.502} & \underline{4.674} & \textbf{0.344}* & 4.716 & \textbf{0.263}* \\
\texttt{AIM} & 3.627 & 0.498 & 2.804 & 0.465 & \underline{2.648} & 0.405 & 6.594 &  \underline{0.355} & 3.845 & 0.299 & 3.641 & 0.274 & 12.193 & 0.504 & 4.705 & \underline{0.417} & \underline{3.866} & \underline{0.343} \\
\texttt{PrivMRF} & \underline{3.387} & \underline{0.488} & \underline{2.748} & \underline{0.391} & 2.996 & \underline{0.384} & \underline{5.829} & \textbf{0.349} & \underline{3.562} & \underline{0.267} & \underline{3.391} & \textbf{0.202}* & \underline{10.124} & \underline{0.503} & \textbf{4.431}* & \underline{0.417} & \textbf{3.352}* & 0.373 \\
\texttt{GEM} & 7.367 & 0.584 & 4.374 & 0.460 & 4.197 & 0.458 & 26.549 & 0.614 & 12.077 & 0.465 & 12.694 & 0.468 & 31.716 & 0.682 & 17.112 & 0.458 & 13.680 & 0.382 \\
\texttt{DP-MERF} & 5.663 & 0.737 & 4.681 & 0.687 & 4.171 & 0.679 & 10.125 & 0.595 & 7.180 & 0.510 & 7.725 & 0.502 & 12.700 & 0.666 & 9.104 & 0.632 & 9.255 & 0.617 \\
\texttt{PrivBayes} & 33.261 & 0.478 & 29.141 & 0.470 & 28.354 & 0.458 & 10.967 & 0.409 & 5.985 & 0.365 & 5.119 & 0.358 & 35.978 & 0.538 & 15.139 & 0.441 & 9.832 & 0.420 \\
\midrule
\textbf{Ground Truth} & 0.046 & 0.140 & 0.046 & 0.140 & 0.046 & 0.140 & 0.053 & 0.346 & 0.053 & 0.346 & 0.053 & 0.346 & 0.410 & 0.066 & 0.410 & 0.066 & 0.410 & 0.066 \\
\bottomrule
\end{tabular}%
}
\vskip -0.0in
\end{table*}

\begin{table*}[htbp]
\centering
\tiny
\setlength{\tabcolsep}{3.5pt}
\caption{Real data comparison for downstream machine learning efficacy ($R^2$).
Best results are in \textbf{bold}; second-best are \underline{underlined}.
The best results that hold significance over the others have a *.}
\label{tab:R^2}
\vskip -0.05in
\begin{tabular}{@{}lccccccccccccccc@{}}
\toprule
& \multicolumn{5}{c}{\textsc{California} } & \multicolumn{5}{c}{\textsc{House-16H}} & \multicolumn{5}{c}{\textsc{Cpu-act}} \\
\cmidrule(lr){2-6} \cmidrule(lr){7-11} \cmidrule(lr){12-16}
\textbf{Method} & Cat & MLP & RF & XGB & SVM & Cat & MLP & RF & XGB & SVM & Cat & MLP & RF & XGB & SVM \\
\midrule
\multicolumn{16}{c}{\textbf{$\varepsilon=0.2$}} \\
\texttt{PrAda-GAN}      & \textbf{0.312}* & \textbf{0.247}*& \textbf{0.285}* & \textbf{0.272}*& \textbf{0.226}* & \textbf{0.139}* & \textbf{0.088}* &\textbf{ 0.086} & \textbf{0.083}*  & \textbf{0.105}* & \textbf{0.133}* & \textbf{-0.010} & \textbf{0.098}*  & \textbf{0.095} & \underline{0.017}\\
\texttt{AIM}       & 0.015 & -0.004 &-0.190 & -0.012 & -0.001 & -0.039 & -0.045 & -0.071 & -0.186 & \underline{-0.002} & 0.003 & -0.166 & 0.039 & -0.037 & -0.078  \\
\texttt{PrivMRF}   & 0.003 & \underline{0.010}  & -0.238& -0.019 & -0.015 & -0.009 & \underline{-0.009} &\underline{ -0.020} & \underline{-0.066} & \underline{-0.002} & -0.001 &-0.208  &  \underline{0.044}  & -0.039  & -0.058   \\
\texttt{GEM}       & -0.152& -0.138 & -0.471 & -0.234 & -0.260 & -0.307 & -0.679 & -0.686 & -1.281 & -0.305 & -0.327 & -0.457 & -0.955 & -0.489 & -0.371 \\
\texttt{DP-MERF}   & \underline{0.201}& -0.166 & \underline{-0.082}& \underline{0.126} & \underline{0.047} & \underline{0.066}  & -0.652  & -1.213 & -0.944  & -0.123 & \underline{0.111} & \underline{-0.066} & 0.002 &\underline{0.042}  & \textbf{0.134}  \\
\texttt{PrivBayes} & 0.024 & -0.040 & -0.088& -0.098& -0.030& -0.007 & -0.314 & -0.083 & -0.237 & -0.026  & -0.142 & -1.289 & -2.779 & -0.877 & -0.004  \\
\midrule
\multicolumn{16}{c}{\textbf{$\varepsilon=1.0$}} \\
\texttt{PrAda-GAN}      & \textbf{0.495}* & \textbf{0.480}* & \textbf{0.434}* & \textbf{0.443}* & \textbf{0.427}* & \textbf{0.240}* & \textbf{0.185}*
 & \textbf{0.198}* & \textbf{ 0.196}* & \textbf{0.160} & \underline{0.127} & \textbf{0.073}* & \textbf{0.103} & \underline{0.106}* & \textbf{0.096}\\
\texttt{AIM}       & 0.005 & \underline{0.017} & -0.212 &-0.038 & -0.009 & \underline{0.157} & \underline{0.124} & \underline{0.076} & \underline{-0.049} & \underline{0.145} & 0.042 & -0.117& \underline{0.101} & 0.012& -0.054\\
\texttt{PrivMRF}   & 0.003&-0.009 & -0.212 & -0.022& -0.013& 0.0004 & -0.004 & -0.028 & -0.058 & -0.002 & 0.024 & -0.195 & 0.075 & -0.003 & -0.071 \\
\texttt{GEM}       & -0.041&-0.079 &-0.345 & -0.087& -0.070& -0.093 & -0.181 & -0.248 & -0.421 & -0.041 & 0.102 & -0.157 & -0.107 & 0.027 & 0.038\\
\texttt{DP-MERF}   & \underline{0.290} & -0.220 & \underline{-0.037} & \underline{0.254} & \underline{0.157} & 0.126 & -0.534 & -0.472 & -0.411  & 0.065 & \textbf{0.149}* & \underline{-0.003}& 0.085 & \textbf{0.112} & \underline{0.090}\\
\texttt{PrivBayes} & 0.019 & -0.017 & -0.075 & -0.088  & -0.005 & 0.005 & -0.450 & -0.003 & -0.114 & -0.018  &0.013 &-0.717 & -3.596 &-0.505 &  -0.033 \\
\midrule
\multicolumn{16}{c}{\textbf{$\varepsilon=5.0$}} \\
\texttt{PrAda-GAN}      & \underline{0.615} & \textbf{0.607}* & \textbf{0.577}& \underline{0.580} &\textbf{0.587}* & 0.261 & \underline{0.205} & 0.225 & \underline{0.221} & \textbf{0.218} & 0.171 & \textbf{0.158}* & \textbf{0.161} & \textbf{0.167}
& \underline{0.090}\\
\texttt{AIM}       & \textbf{0.642}* & \underline{0.489} & \underline{0.567} & \textbf{0.583} & 0.527 & \textbf{0.343}* & 0.184 & \textbf{0.320}* & \textbf{0.277}* & \underline{0.207} & \textbf{0.212}* & \underline{-0.005} & \underline{0.151} & \underline{0.165} & \textbf{0.150}* \\
\texttt{PrivMRF}   & 0.057& -0.046 &-0.166 & 0.010 & 0.107 & \underline{0.317} &\textbf{0.232}* &  \underline{0.305} & 0.185 & 0.197 & 0.037 & -0.246 & 0.079 & 0.006 & -0.074\\
\texttt{GEM}       & -0.029& -0.049 & -0.343& -0.093 & -0.046 & -0.042 & -0.102 & -0.124 & -0.364 & -0.006  & \underline{0.179} & -0.023
& 0.112 & 0.145 & 0.067 \\
    \texttt{DP-MERF}   &0.298 & -0.078& 0.173& 0.300 & 0.021& 0.112 & -0.573 & -0.044 & -0.356 & 0.088 & 0.174 & -0.327 & 0.148 & 0.143 & 0.088
 \\
\texttt{PrivBayes} & 0.459 & 0.398 & 0.419 & 0.405 & 0.469 & 0.014 & -0.708 & -0.011 & -0.093 & -0.021 &0.020 & -0.497& -3.620 & -0.408 & -0.046\\
\cmidrule[0.5pt]{1-16}
\textbf{Ground Truth} & 0.856 & 0.809 & 0.812 & 0.818  & 0.657  & 0.548 & 0.524 & 0.540 & 0.510 & 0.309  &0.979 & 0.952 & 0.977 & 0.968 & 0.203 \\
\bottomrule
\end{tabular}
\vskip -0.1in
\end{table*}

\paragraph{Parameter Analysis of Penalty Parameters} 
Using the same functions \( f_j \), we investigate the effects of the penalty parameters \( \lambda_j \) for \( j \in [d] \).  
To reduce the complexity of the search, we parameterize the penalties as \( \lambda_j = \lambda \cdot j^{\gamma} \) for \( j = 1, \dots, d \), and focus on varying the global parameters \( \lambda \) and \( \gamma \).  
As shown in the top panel of Figure~\ref{fig:lambda_gamma}, there exists an optimal value of \( \lambda \) for each noise level \( \sigma \), and this optimal value increases as \( \sigma \) becomes larger.  
This observation aligns with the theory, where the optimal choice of \( \lambda_j \) in equations~\eqref{equ:def-of-psi} and~\eqref{equ:def-of-psi-phase-2} increases with the optimization error, which itself grows with larger noise levels \( \sigma \).  
In the bottom panel of Figure~\ref{fig:lambda_gamma}, we observe that in most cases, setting \( \gamma = 0 \) yields a sufficiently low error, indicating that the magnitude of the penalty need not depend on the number of preconditioned features.  
This phenomenon is also consistent with Theorems~\ref{thm:feature-selection} and~\ref{thm:estimation}.

\begin{figure}[!h]
\vskip -0.0in
    \centering
    \begin{minipage}{\columnwidth}
        \centering
        \includegraphics[width=0.75\linewidth]{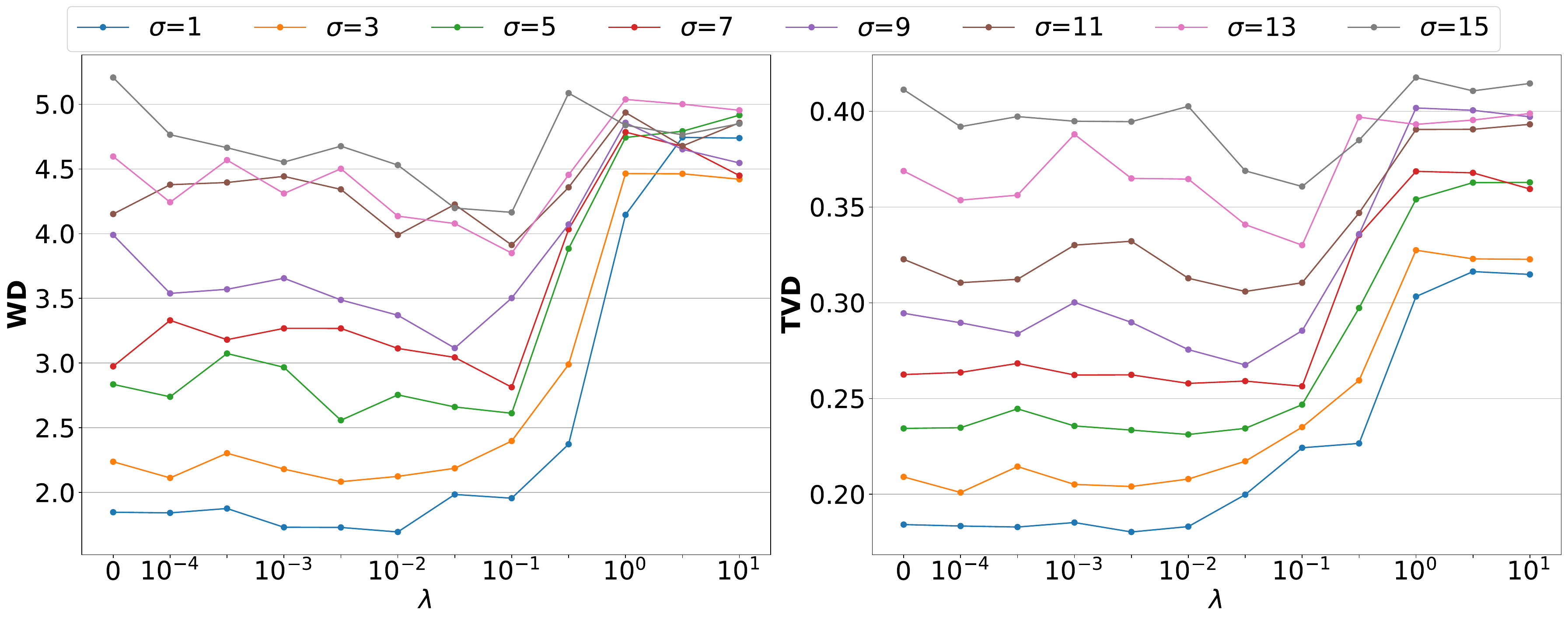}
        \label{fig:lambda}
    \end{minipage}
    \\  
    \begin{minipage}{\columnwidth}
        \centering
\includegraphics[width=0.75\linewidth]{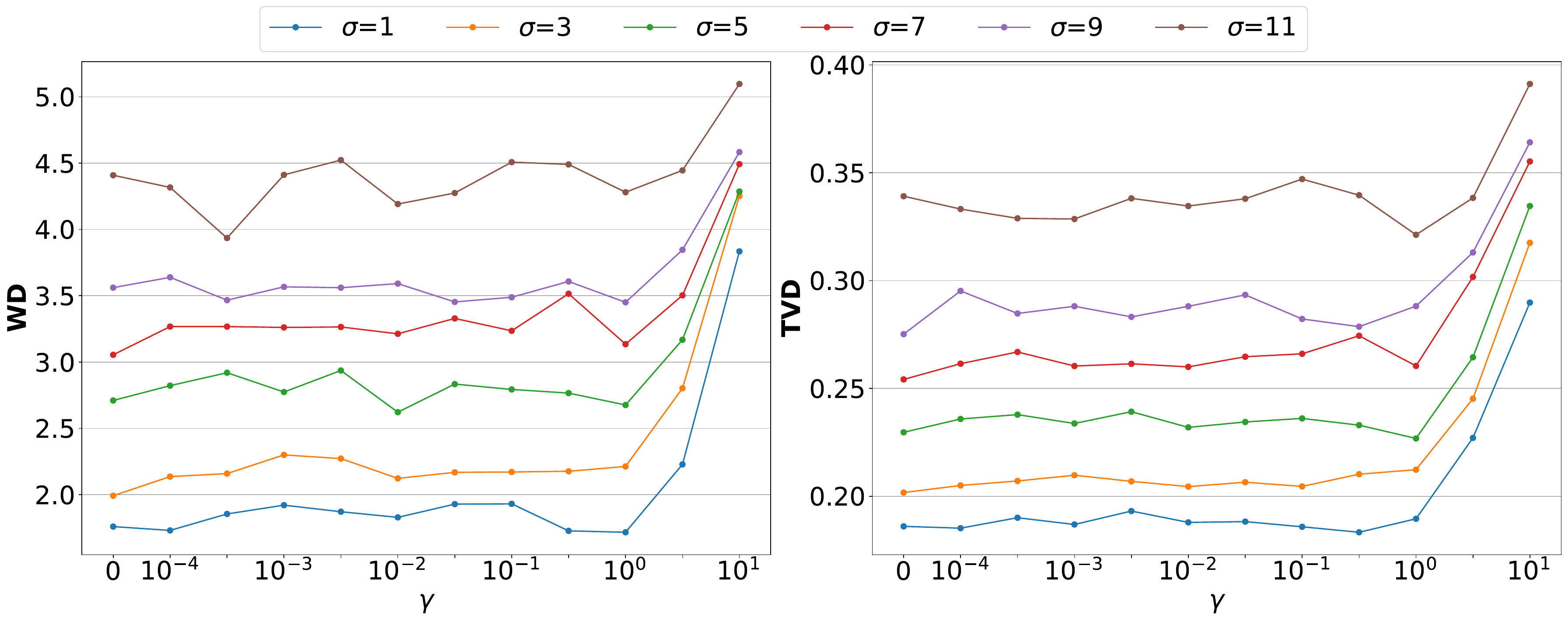}
        \label{fig:gamma}
    \end{minipage}
    \vskip -0.1in
    \caption{Average WD and TVD under varying $\lambda$ and $\gamma$.}
    \label{fig:lambda_gamma}
    \vskip -0.2in
\end{figure}

\subsection{Real Data Comparison}
\label{real_data}

\paragraph{Experiment Setup} 
In our benchmarking suite, we compare \texttt{PrAda-GAN} against several baseline methods, including Bayesian network–based approaches such as \texttt{AIM} \citep{mckenna2022aim}, \texttt{PrivMRF} \citep{cai2021data}, and \texttt{PrivBayes} \citep{zhang2017privbayes}, as well as state-of-the-art deep learning methods for differentially private synthetic data generation, including \texttt{GEM} \citep{liu2021iterative} and \texttt{DP-MERF} \citep{harder2021dp}.
We select three commonly used machine learning datasets, \textsc{California}, \textsc{House-16H}, and \textsc{Cpu-act}, from OpenML \citep{OpenML2013}, all of which contain continuous features and targets. Descriptive statistics for these datasets are provided in Table~\ref{tab:realdatainfo} in the appendix.
Each experiment was repeated over 10 random trials, with randomness stemming from initialization, data partitioning, and shuffling.  
In each trial, the dataset was randomly split into training and testing sets in 6:4.  
For downstream tasks, model hyperparameters were predetermined using a validation set held out from the training split.

\paragraph{Distribution Similarity}
The representative results on real-world datasets for privacy budgets $\varepsilon = 0.2, 1.0,$ and $5.0$ are presented in Table \ref{tab:similarity_results}. Additional evaluation results using other metrics are provided in the appendix. As shown in the table, \texttt{PrAda-GAN} achieves superior performance compared to competing methods in most scenarios. Among the baseline methods, \texttt{AIM} and \texttt{PrivMRF} emerge as the strongest competitors, while other deep learning-based methods fail to achieve comparable performance.

\paragraph{Machine Learning Efficacy}
A corresponding set of experiments evaluating machine learning efficacy is presented in Table~\ref{tab:R^2} for \( R^2 \) scores, and in Table~\ref{tab:rmse} in the appendix for RMSE.  
Our method consistently achieves higher utility across all scenarios, and demonstrates substantial advantages over the competitors under low privacy budgets, i.e., \( \varepsilon = 0.2 \) and \( 1.0 \).

\section{Conclusion}
\label{sec:con}

This paper addresses the limitations of existing methods for differentially private tabular data synthesis by introducing a novel generative framework, \texttt{PrAda-GAN}.  We provide theoretical analysis of \texttt{PrAda-GAN}, including a bound on the distance to the optimal generator and a generalization bound for the Wasserstein distance between the generated and true data distributions. Our results highlight the role of adaptive regularization in improving convergence. Empirical evaluations on real-world datasets show that \texttt{PrAda-GAN} outperforms existing methods across multiple metrics, achieving strong utility while ensuring rigorous privacy guarantees. 

\section{Acknowledgments}
The authors would like to thank the SPC and the reviewers for their constructive comments and recognition of this work. This work is supported by National Natural Science Foundation of China (72371241), the MOE Project of Key Research Institute of Humanities and Social Sciences (22JJD910001); the Big Data and Responsible Artificial Intelligence for National Governance, Renmin University of China; Public Computing Cloud,
Renmin University of China.

\bibliography{aaai2026}

\appendix

\newpage

\onecolumn

\setcounter{secnumdepth}{2}

\section{An Example of Bayes Network}\label{app:bayesian network}
\begin{example}\label{ex:bayesnetwork}
    Consider a Bayes network over a set of continuous variables $\{X^1, X^2, X^3, X^4, X^5\}$ representing \texttt{age}, \texttt{education\_years}, \texttt{salary}, \texttt{working\_hours}, and \texttt{credit\_score}, respectively. A possible Bayes network structure is as follows:
    \begin{itemize}
        \item $\Pi_1 = \emptyset$ (age is a root node)
        \item $\Pi_2 = \{1\}$ (education\_years depends on age)
        \item $\Pi_3 = \{1, 2\}$ (salary depends on age and education\_years)
        \item $\Pi_4 = \{3\}$ (working\_hours depends on salary)
        \item $\Pi_5 = \{3, 4\}$ (credit\_score depends on salary and working\_hours)
    \end{itemize}
    This structure satisfies Assumption~\ref{asp:bayesnetwork}, where each variable $X^j$ only depends on its parent set $\Pi_j$, and the graph is acyclic.
\end{example}

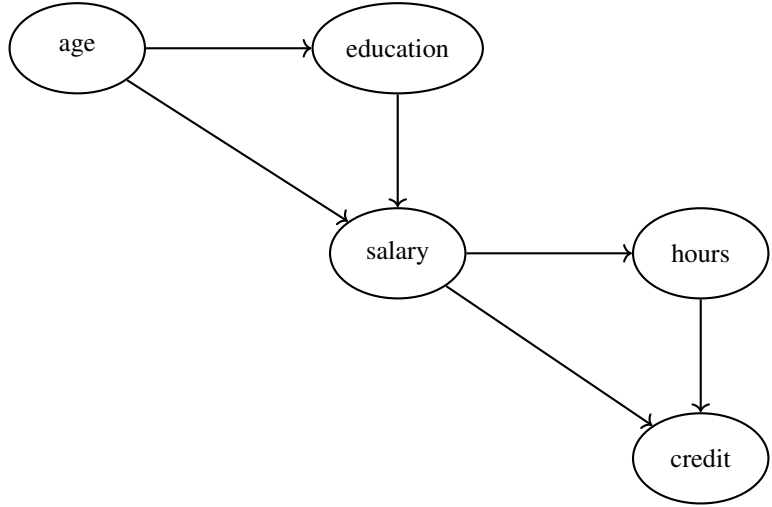
\begin{figure}[h]
    \centering
    \begin{tikzpicture}[
        node distance=1.5cm and 2.2cm,
        every node/.style={shape=ellipse, draw, minimum width=1.8cm,
            minimum height=1.2cm, },
        every path/.style={->, thick}
        ]
        \node (age) {age};
        \node[right=of age] (edu) {education};
        \node[below=of edu] (salary) {salary};
        \node[right=of salary] (hours) {hours};
        \node[below=of hours] (credit) {credit};

        \draw (age) -- (edu);
        \draw (age) -- (salary);
        \draw (edu) -- (salary);
        \draw (salary) -- (hours);
        \draw (salary) -- (credit);
        \draw (hours) -- (credit);
    \end{tikzpicture}
    \caption{An example Bayes network over continuous variables.}
\end{figure}

\section{Methodologies}

\subsection{A Two-step Algorithms}

In this section, we propose a two-step version of \texttt{PrAda-GAN} that is convenient for theoretical analysis. 
The algorithm is presented in Algorithm \ref{alg:privdag-gan-two-step}.
Note that there are commented-out codes after line 9, which were included for theoretical considerations but may not be used in practice.

\begin{algorithm}[htbp]
   	\caption{\texttt{PrAda-GAN}}
   	\label{alg:privdag-gan-two-step}
    {\bfseries Input: }{ Private data ${D} = \{X_i\}_{i=1}^n$.   }\\
     {\bfseries Parameters: }{ Learning rate $\eta_{\bm \theta}, \eta_{\bm \nu}$. Number of iterations $T$, relative number of iterations $t_g$.  
     Threshold $\tau$. 
     Regularizations $\lambda_1,\ldots, \lambda_d$. Privacy noise level $\sigma$.  }\\
     {\bfseries Initialization: }{ Initial parameters $\bm \theta^1$, $\bm \nu^1$, \texttt{\# $\overline{\bm \theta}^1$, $\overline{\bm \nu}^1$}.   }\\
     \For{$t$ in $[T]$}{
    \texttt{\# Update $\nu$ every step.  }\\
     ${\bm \nu}^{t+1}={\bm \nu}^{t} + \eta_{\nu}\cdot $ \texttt{PrivGrad}$(\nabla_{\bm \nu} \tilde{\Delta}({\bm \theta}^{t}, {\bm\nu}^{t}, D), \sigma)$  \\
    \texttt{\# Update $\theta$ every $t_g$ steps.  }
     \\
     ${\bm \theta}^{t+1} = {\bm \theta}^t - \eta_{\theta} \cdot \eins \left(t\mod t_g \equiv 0\right) \cdot \nabla_{\bm \theta} \tilde{\Delta}({\bm \theta}^{t}, {\bm\nu}^{t+1}, D)$. 
    }
    \texttt{\#} \ \ 
    \For{$j\in [d]$}{
\texttt{\#} \ \  \For{$k \in [j-1]$}{
    \texttt{\# Rule out the unimportant variables.  }\\
\texttt{\#} \ \  \If{$\left\|(\bm W^{T+1}_j)^{k,:}\right\|_2\leq \tau$}{
\texttt{\#} \ \ 
$(\overline{\bm W}_j^{T+1})^{j', :} = \mathbf{0}$, $(\overline{\bm W}_j^{T+1})^{j', :}$.requires\_grad = \texttt{False}. 
}
} }

     \texttt{\#} \ \ \For{$t$ in $[T]$}{
    \texttt{\# Re-optimize.  }\\
     \texttt{\#} \ \  $\overline{\bm \nu}^{t+1}=\overline{\bm \nu}^{t} + \eta_{\nu}\cdot $ \texttt{PrivGrad}$(\nabla_{{\bm \nu}} {\Delta}(\overline{\bm \theta}^{t}, \overline{\bm\nu}^{t}, D), \sigma)$ \\
     \texttt{\#} \ \  $\overline{\bm \theta}^{t+1} = \overline{\bm \theta}^t - \eta_{\theta} \cdot \eins \left(t\mod t_g \equiv 0\right) \cdot \nabla_{{\bm \theta}} {\Delta}(\overline{\bm \theta}^{t}, \overline{\bm\nu}^{t+1}, D)$.  }

    {\bfseries Output: }{Generator $g_{\bm \theta^{T+1}}$ and discriminator $f_{\bm \nu^{T+1}}$,  \texttt{\#  $g_{\overline{\bm \theta}^{T+1}}$ and $f_{\overline{\bm \nu}^{T+1}}$.}} 
   \end{algorithm} 

\section{Proofs}\label{app:proof}

\subsection{Proof of Theorem \ref{thm:feature-selection}}

\begin{proof}[\textbf{Proof of Theorem \ref{thm:feature-selection}}]
We first focus on the quantity $\mathbb{E}\left[d\left({\boldsymbol{\theta}}^{T+1}, \Theta\right)\right]$, which is defined as follows. 
Define the collection 
\begin{align}
\Theta = \left\{{\bm \theta} \in \mathcal{G} \mid g_{\bm \theta} \in {\arg \min}_{\bm \theta } \mathcal{W} \left( \mathrm{P}_{g_{\bm \theta}(Z)}, \mathrm{P}_X\right)\right\}, 
\end{align}
meaning that all the feasible parameters that achieve the minimum Wasserstein distance in $\mathcal{G}$. 
Define 
\begin{align*}
    \widehat{\bm \theta} \in {\arg \min}_{{\bm \theta}\in \Theta} d({\bm \theta}^{T+1}, {\bm \theta}), 
\end{align*}
which is an arbitrary element in $\mathcal{\theta}$ that is closest to ${\bm \theta}^{T+1}$. 
Under Assumption \ref{asp:analytic}, we can use Lojasewicz’s inequality to show that there exists a constant $a>2$ such that 
\begin{align}\label{equ:proof-lowerbound-w-by-d}
    d\left({\boldsymbol{\theta}}^{T+1}, \Theta\right)^{a} \lesssim \mathcal{W} \left( \mathrm{P}_{g_{{\bm \theta}^{T+1}}(Z)}, \mathrm{P}_X\right) - \mathcal{W} \left( \mathrm{P}_{g_{\widehat{\bm \theta}}(Z)}, \mathrm{P}_X\right).
\end{align}
See a similar reasoning in \citet{wang2024penalized}. 
This notion is also analogous to the Tsybakov noise condition, and see \citet{ramdas2012optimal, xu2025beyond} for related discussions. 
Then, let the class of 1-Lipschitz functions from $\mathbb{R}^d$ to $\mathbb{R}$ be $\mathcal{F}_{Lip, 1}$. 
By the property of Wasserstein distance \citep{villani2008optimal, arjovsky2017wasserstein}, there hold
\begin{align*}
    \mathcal{W} \left( \mathrm{P}_{g_{{\bm \theta}^{T+1}}(Z)}, \mathrm{P}_X\right) = \sup_{f \in \mathcal{F}_{Lip, 1}} \left(
    \mathbb{E}_{X} \left[f(X)\right] - \mathbb{E}_{Z} \left[f(g_{{\bm \theta}^{T+1}}(Z))\right]\right)
\end{align*}
as well as 
\begin{align*}
    \mathcal{W} \left( \mathrm{P}_{g_{\widehat{\bm \theta}}(Z)}, \mathrm{P}_X\right) = \sup_{f \in \mathcal{F}_{Lip, 1}} \left(
    \mathbb{E}_{X} \left[f(X)\right] - \mathbb{E}_{Z} \left[f(g_{\widehat{\bm \theta}}(Z))\right]\right)
\end{align*}
We define two abbreviations for notation simplicity. 
\begin{align*}
    R({\bm \theta}, f) := \mathbb{E}_{X} \left[f(X)\right] - \mathbb{E}_{Z} \left[f(g_{{\bm \theta}}(Z))\right],
\end{align*}
\begin{align*}
    R_n({\bm \theta}, f) := \frac{1}{n}\sum_{i=1}^{n}f(X_i) - \frac{1}{n_g}\sum_{i=1}^{n_g} f(g_{{\bm \theta}}(Z)). 
\end{align*}
Then, bringing the above equations into \eqref{equ:proof-lowerbound-w-by-d}, we get 
\begin{align}
    \mathbb{E}\left[d\left({\boldsymbol{\theta}}^{T+1}, \Theta\right)^{a} \right]\lesssim &  \mathbb{E}\left[\mathcal{W} \left( \mathrm{P}_{g_{{\bm \theta}^{T+1}}(Z)}, \mathrm{P}_X\right) - \mathcal{W} \left( \mathrm{P}_{g_{\widehat{\bm \theta}}(Z)}, \mathrm{P}_X\right)\right] \\
    = & \mathbb{E}\left[\sup_{f \in \mathcal{F}_{Lip, 1}} \left(R(\bm \theta^{T+1}, f)\right) - \sup_{f \in \mathcal{F}_{Lip, 1}} \left(R(\widehat{\bm \theta}, f)\right)\right].
    \label{equ:proof-two-terms}
\end{align}
Here, the outside $\mathbb{E}$ is taken with respect to the randomness contained in $\bm \theta^{T+1}$ and its projection $\widehat{\bm \theta}$. 

We deal with the two terms separately in \eqref{equ:proof-two-terms}.
For the first term, we have 
\begin{align}\label{equ:proof-two-terms-plus}
    \sup_{f \in \mathcal{F}_{Lip, 1}} \left(R(\bm \theta^{T+1}, f)\right) \leq \sup_{f \in \mathcal{F}_{Lip, 1}} \left(R(\bm \theta^{T+1}, f) - R_n(\bm \theta^{T+1}, f)\right) + \sup_{f \in \mathcal{F}_{Lip, 1}} \left(R_n(\bm \theta^{T+1}, f)\right).  
\end{align}
By a similar argument as in the Lemmas 1 and 2 in \citet{wang2024penalized} (chaining + applying Proposition 3.1 in \citet{lu2020universal}), we have 
\begin{align}\label{equ:proof-uniform-converge-on-l1-function-class}
    \mathbb{E}\left[\sup_{f \in \mathcal{F}_{Lip, 1}} \left(R(\bm \theta^{T+1}, f) - R_n(\bm \theta^{T+1}, f)\right) \right] \lesssim (\sqrt{d} + \log n) n^{- \frac{1}{d}}. 
\end{align}
We also have 
\begin{align}
    \sup_{f \in \mathcal{F}_{Lip, 1}} \left(R_n(\bm \theta^{T+1}, f)\right) = & \sup_{f \in \mathcal{F}_{Lip, 1}} \left(R_n(\bm \theta^{T+1}, f)\right) -  R_n(\bm \theta^{T+1}, f_{\bm \nu^{T+1}}) + R_n(\bm \theta^{T+1}, f_{\bm \nu^{T+1}}) \\
    \leq & \sup_{f \in \mathcal{F}_{Lip, 1}}  \inf_{f' \in \mathcal{F}} \left(R_n(\bm \theta^{T+1}, f) -  R_n(\bm \theta^{T+1}, f')\right) + R_n(\bm \theta^{T+1}, f_{\bm \nu^{T+1}}) \\
    \leq & \sup_{f \in \mathcal{F}_{Lip, 1}}  \inf_{f' \in \mathcal{F}} \left\|f - f'\right\|_{\infty} + R_n(\bm \theta^{T+1}, f_{\bm \nu^{T+1}}). 
    \label{equ:proof-derive-approximiation-error}
\end{align}
Denote one of the solutions of $\inf_{{\bm \nu}\in \mathcal{F}} \sup_{{\bm \theta}\in \mathcal{G}}\tilde{\Delta}({\bm \theta}, {\bm \nu}, D)$ as $\overline{\bm \theta}$ and $\overline{\bm \nu}$.  
By Assumption \ref{asp:private-optimization}, we have 
\begin{align}\label{equ:proof-use-of-assumption-optimizaiton}
R_n(\bm \theta^{T+1}, f_{\bm \nu^{T+1}}) \leq &  R_n(\overline{\bm \theta}, f_{\overline{\bm \nu}})  +   \sum_{j = 1}^d \lambda_{j} L(\overline{\bm W}_j) - \sum_{j = 1}^d \lambda_{j} L({\bm W}^{T+1}_j) +  \mathcal{E} ,
\end{align}
where we refer to the $\bm W$s in $\bm \theta^{T+1}$ and $\overline{\bm \theta}$ as the ones with the same superscript. 
Then, by the definition of $\overline{\bm \theta}$, there holds 
\begin{align}\label{equ:proof-use-of-def-of-overline-theta}
R_n(\bm \theta^{T+1}, f_{\bm \nu^{T+1}}) \leq R_n(\widehat{\bm \theta}, f_{\overline{\bm \nu}})  +   \sum_{j = 1}^d \lambda_{j} L(\widehat{\bm W}_j) - \sum_{j = 1}^d \lambda_{j} L({\bm W}^{T+1}_j) +  \mathcal{E}  
\end{align}
Then, bringing \eqref{equ:proof-uniform-converge-on-l1-function-class}, \eqref{equ:proof-derive-approximiation-error}, and \eqref{equ:proof-use-of-def-of-overline-theta} into \eqref{equ:proof-two-terms-plus} leads to 
\begin{align}\nonumber
    \sup_{f \in \mathcal{F}_{Lip, 1}} \left(R(\bm \theta^{T+1}, f)\right) \lesssim & (\sqrt{d} + \log n) n^{- \frac{1}{d}} + \sup_{f \in \mathcal{F}_{Lip, 1}}  \inf_{f' \in \mathcal{F}} \left\|f - f'\right\|_{\infty}  \\
    & + R_n(\widehat{\bm \theta}, f_{\overline{\bm \nu}})  +   \sum_{j = 1}^d \lambda_{j} L(\widehat{\bm W}_j) - \sum_{j = 1}^d \lambda_{j} L({\bm W}^{T+1}_j) +  \mathcal{E}.
    \label{equ:proof-result-of-two-terms-plus}
\end{align}

Next, we consider the second term, which is $\sup_{f \in \mathcal{F}_{Lip, 1}} \left(R(\widehat{\bm \theta}, f)\right)$. 
Again, we first have 
\begin{align}\label{equ:proof-two-terms-minus}
    \sup_{f \in \mathcal{F}_{Lip, 1}} \left(R(\widehat{\bm \theta}, f)\right) \geq  \sup_{f \in \mathcal{F}_{Lip, 1}} \left(R_n(\widehat{\bm \theta}, f)\right) - \sup_{f \in \mathcal{F}_{Lip, 1}} \left(R_n(\widehat{\bm \theta}, f) - R(\widehat{\bm \theta}, f)\right).
\end{align}
Similar to \eqref{equ:proof-uniform-converge-on-l1-function-class}, we have
\begin{align*}
    \sup_{f \in \mathcal{F}_{Lip, 1}} \left(R_n(\widehat{\bm \theta}, f) - R(\widehat{\bm \theta}, f)\right) \lesssim (\sqrt{d} + \log n) n^{- \frac{1}{d}}. 
\end{align*}
Also, there holds 
\begin{align*}
    \sup_{f \in \mathcal{F}_{Lip, 1}} \left(R_n(\widehat{\bm \theta}, f)\right)  = & \sup_{f \in \mathcal{F}_{Lip, 1}} \left(R_n(\widehat{\bm \theta}, f)\right) - \sup_{f \in \mathcal{F}} \left(R_n(\widehat{\bm \theta}, f)\right)  + \sup_{f \in \mathcal{F}} \left(R_n(\widehat{\bm \theta}, f)\right) \\
    \geq & - \sup_{f \in \mathcal{F}_{Lip, 1}}  \inf_{f' \in \mathcal{F}} \left\|f - f'\right\|_{\infty} + \sup_{f \in \mathcal{F}} \left(R_n(\widehat{\bm \theta}, f)\right) \\
    \geq & - \sup_{f \in \mathcal{F}_{Lip, 1}}  \inf_{f' \in \mathcal{F}} \left\|f - f'\right\|_{\infty} + R_n(\widehat{\bm \theta}, f_{\overline{\bm \nu}}). 
\end{align*}

Then, bringing these results into \eqref{equ:proof-two-terms-minus} leads to
\begin{align}
-  \sup_{f \in \mathcal{F}_{Lip, 1}} \left(R(\widehat{\bm \theta}, f)\right) \lesssim 
(\sqrt{d} + \log n) n^{- \frac{1}{d}} +\sup_{f \in \mathcal{F}_{Lip, 1}}  \inf_{f' \in \mathcal{F}} \left\|f - f'\right\|_{\infty} - R_n(\widehat{\bm \theta}, f_{\overline{\bm \nu}}),
    \label{equ:proof-result-of-two-terms-minus}
\end{align}
which together with \eqref{equ:proof-result-of-two-terms-plus} lead to 
\begin{align}\label{equ:proof-before-get-d-a}
     & \mathbb{E}\left[d\left({\boldsymbol{\theta}}^{T+1}, \Theta\right)^{a} \right]
     \\
     \lesssim & (\sqrt{d} + \log n) n^{- \frac{1}{d}} +\sup_{f \in \mathcal{F}_{Lip, 1}}  \inf_{f' \in \mathcal{F}} \left\|f - f'\right\|_{\infty} + \sum_{j = 1}^d \lambda_{j} L(\widehat{\bm W}_j) - \sum_{j = 1}^d \lambda_{j} L({\bm W}^{T+1}_j) +  \mathcal{E}. \nonumber
\end{align}
Recall the definition of $\bm W$s, which has 
\begin{align*}
  \sum_{j = 1}^d \lambda_{j} L(\widehat{\bm W}_j) - \sum_{j = 1}^d \lambda_{j} L({\bm W}^{T+1}_j) = & \sum_{j = 1}^d  \lambda_j\sum_{k = 1}^{j-1} \|(\widehat{\bm W}_j)^{k, :}\|_2 -   \sum_{j = 1}^d  \lambda_j\sum_{k = 1}^{j-1} \|({\bm W}_j^{T+1})^{k, :}\|_2 \\
  \leq & \sum_{j = 1}^d  \lambda_j\sum_{k = 1}^{j-1} \|(\widehat{\bm W}_j)^{k, :} - ({\bm W}_j^{T+1})^{k, :}\|_2 \\
  \leq & \sum_{j = 1}^d  \lambda_j \sqrt{j} \|\widehat{\bm W}_j - {\bm W}_j^{T+1}\|_2 \leq \sum_{j = 1}^d  \lambda_j \sqrt{j} \|\widehat{\bm \theta}_j - {\bm \theta}_j^{T+1}\|_2. 
\end{align*}
Applying Cauchy's inequality and subsequently Young's inequality, we have 
\begin{align}
   \mathbb{E}\left[ \sum_{j = 1}^d \lambda_{j} L(\widehat{\bm W}_j) - \sum_{j = 1}^d \lambda_{j} L({\bm W}^{T+1}_j)\right] \leq &   \mathbb{E}\left[ \sum_{j = 1}^d  \lambda_j \sqrt{j} \|\widehat{\bm \theta}_j - {\bm \theta}_j^{T+1}\|_2 \right]\\
    \leq & \mathbb{E}\left[ \|\widehat{\bm \theta} - {\bm \theta}^{T+1}\|_2 \sqrt{\sum_{j=1}^d \lambda_j^2 j} \right] \\
    \leq & \mathbb{E}\left[ \frac{\|\widehat{\bm \theta} - {\bm \theta}^{T+1}\|_2^a }{a} + \frac{\left(\sqrt{\sum_{j=1}^d \lambda_j^2 j}\right)^{\frac{a}{a-1}} }{(a - 1) / a}\right].
    \label{equ:proof-apply-young}
\end{align}
This brings \eqref{equ:proof-before-get-d-a} into 
\begin{align*}
    \mathbb{E}\left[d\left({\boldsymbol{\theta}}^{T+1}, \Theta\right)^{a} \right]
     = &     \mathbb{E}\left[\|\widehat{\bm \theta} - {\bm \theta}^{T+1}\|_2^a \right]\\
     \lesssim & (\sqrt{d} + \log n) n^{- \frac{1}{d}} +\sup_{f \in \mathcal{F}_{Lip, 1}}  \inf_{f' \in \mathcal{F}} \left\|f - f'\right\|_{\infty} +    \mathcal{E}\\
     & +
\mathbb{E}\left[ \frac{\|\widehat{\bm \theta} - {\bm \theta}^{T+1}\|_2^a }{a}\right] +   \left(\sum_{j=1}^d \lambda_j^2 j \right)^{\frac{a}{2(a-1)}}.
\end{align*}
Thus, some algebra leads to 
\begin{align}\label{equ:proof-final-result-of-d}
    \mathbb{E}\left[d\left({\boldsymbol{\theta}}^{T+1}, \Theta\right)^{a} \right] \lesssim (\sqrt{d} + \log n) n^{- \frac{1}{d}} +\sup_{f \in \mathcal{F}_{Lip, 1}}  \inf_{f' \in \mathcal{F}} \left\|f - f'\right\|_{\infty} +    \mathcal{E} + \left(\sum_{j=1}^d \lambda_j^2 j \right)^{\frac{a}{2(a-1)}}.
\end{align}

We proceed to show the induced sparsity of $\bm W$. 
Denote the function $\kappa({\bm W}_j)$ as setting the entries of ${\bm W}_j^{\Pi_j^c}$ to be 0. 
Also, denote $\kappa\left(\bm \theta\right)$ as replacing $\bm W$ in $\bm \theta$ by $\kappa(\bm W)$. 
As stated in the literature, if $\bm \theta \in \Theta$, then $\kappa(\bm \theta)$ too. 
See, for instance, Lemma 6 in \citet{wang2024penalized} and Lemma 3.1 in \citet{dinh2020consistent}. 
Thus, we start from 
\begin{align}\label{equ:proof-variable-selection-oracle}
 R_n({\bm \theta}^{T+1}, f_{{\bm \nu}^{T+1}})  +   \sum_{j = 1}^d \lambda_{j} L({\bm W}_j^{T+1}) \leq & 
    R_n(\overline{\bm \theta}, f_{\overline{\bm \nu}})  +   \sum_{j = 1}^d \lambda_{j} L(\overline{\bm W}_j) + \mathcal{E}
    \\ 
    \leq & 
    R_n(\kappa(\widehat{\bm \theta}), f_{\overline{\bm \nu}})  +   \sum_{j = 1}^d \lambda_{j} L(\kappa(\widehat{\bm W}_j)) + \mathcal{E}. 
    \nonumber
\end{align}
By definition, 
\begin{align}
   \sum_{j = 1}^d \lambda_{j}  L(\kappa(\widehat{\bm W}_j)) = \sum_{j = 1}^d  \lambda_j\sum_{k \in \Pi_j} \|(\widehat{\bm W}_j)^{k, :}\|_2. 
\end{align}
Thus,  \eqref{equ:proof-variable-selection-oracle} leads to 
\begin{align}
    \mathbb{E}\left[ \sum_{j=1}^d \lambda_{j} \sum_{k \in \Pi_j^c }\left\| \left({\bm W}_j^{T+1}\right)^{k}\right\|_2\right] \leq & R_n(\kappa(\widehat{\bm \theta}), f_{\overline{\bm \nu}})  - R_n({\bm \theta}^{T+1}, f_{{\bm \nu}^{T+1}}) \\
     + &
    \sum_{j = 1}^d  \lambda_j\sum_{k \in \Pi_j} \|(\widehat{\bm W}_j)^{k, :}\|_2 - \sum_{j = 1}^d  \lambda_j\sum_{k \in \Pi_j} \|({\bm W}_j^{T+1})^{k, :}\|_2\\
    \leq & R_n(\widehat{\bm \theta}, f_{\overline{\bm \nu}})  - R_n({\bm \theta}^{T+1}, f_{{\bm \nu}^{T+1}}) + \mathbb{E}\left[ \|\widehat{\bm \theta} - {\bm \theta}^{T+1}\|_2 \sqrt{\sum_{j=1}^d \lambda_j^2 j} \right],
    \label{equ:proof-derive-single-bound-W}
\end{align}
where the last line is due to the definition of $\kappa$ and a similar reasoning in \eqref{equ:proof-apply-young}. 
With a similar reasoning in \eqref{equ:proof-uniform-converge-on-l1-function-class} and \eqref{equ:proof-derive-approximiation-error}, we get 
\begin{align*}
    R_n(\widehat{\bm \theta}, f_{\overline{\bm \nu}})  - R_n({\bm \theta}^{T+1}, f_{{\bm \nu}^{T+1}}) \lesssim (\sqrt{d} + \log n) n^{- \frac{1}{d}} + \sup_{f \in \mathcal{F}_{Lip, 1}}  \inf_{f' \in \mathcal{F}} \left\|f - f'\right\|_{\infty} + \mathcal{E} = \psi_{n, d}. 
\end{align*}
Remind that we took $\lambda_j =  d^{-1} \psi_{n,d}^{\frac{1}{2}}$. 
Then, 
\begin{align*}
    \sum_{j=1}^d \lambda_j^2 j = \frac{1}{d^2} \frac{d (d+1)}{2} \psi_{n,d}^{\frac{2}{2}} \asymp \psi_{n,d}
\end{align*}
and \eqref{equ:proof-derive-single-bound-W} becomes
\begin{align*}
    \mathbb{E}\left[  \sum_{k \in \Pi_j^c }\left\| \left({\bm W}_j^{T+1}\right)^{k}\right\|_2\right] \lesssim &  \lambda_{j}^{-1} \left[\psi_{n, d} + \left(\psi_{n,d}^{\frac{1}{a}} + \left(\sum_{j=1}^d \lambda_j^2 j \right)^{\frac{1}{2(a-1)}} \right)\sqrt{\sum_{j=1}^d \lambda_j^2 j} \right]\\
    \asymp &  \lambda_{j}^{-1} \psi_{n,d}^{\frac{a}{2 (a-1)}}  =  d \cdot \psi_{n,d}^{\frac{1}{2(a-1)}},
\end{align*}
which also utilized \eqref{equ:proof-final-result-of-d}. 
Similarly, this choice of $\lambda_j$ leads to 
\begin{align*}
    \mathbb{E}\left[d\left({\boldsymbol{\theta}}^{T+1}, \Theta\right) \right] \lesssim \psi_{n,d}^{\frac{1}{a}} + \psi_{n,d}^{\frac{1}{2(a-1)}} \asymp\psi_{n,d}^{\frac{1}{2(a-1)}}. 
\end{align*}

\end{proof}

\begin{lemma}[\textbf{Exact Selection}]\label{lem:exact-selection}
Suppose the conditions in Theorem \ref{equ:thm-bound-of-feature-selection} hold. 
If $d = \log^{c_1} d$, then there exists a positive threshold $\tau$ such that, for all $j\in[d]$ and $k\in [j-1]$,
\begin{align}
    \mathrm{P} \left( \left\|(\bm W^{T+1}_j)^{k,:}\right\|_2 = 0 \iff k \in \Pi_j \right) \geq 1 - 2 \psi_{n,d}^{\frac{1}{4(a-1)}}.
\end{align}
    
\end{lemma}

\begin{proof}[\textbf{Proof of Lemma \ref{lem:exact-selection}}]
We first show that for all ${\bm\theta}\in \Theta$, there exists a $\widehat{\tau}$ such that
\begin{align*}
     \left\|(\widehat{\bm W}^{T+1}_j)^{k,:}\right\|_2 \geq  \widehat{\tau} \iff k \in \Pi_j 
\end{align*}
If no such $\widehat{\tau}$ exists, then there exists a $(j,k)$ such that $\left\|(\widehat{\bm W}^{T+1}_j)^{k,:}\right\|_2\to 0$. 
Suppose $j \in \Pi_{j'}$. 
Then, there exists a constant $c_W$, 
\begin{align*}
   \mathcal{W}\left(  \mathrm{P}_{g_{\widehat{\bm \theta}^{T+1}}(Z)},   \mathrm{P}_{g_{\tau(\widehat{\bm \theta}^{T+1})}(Z)} \right) \geq c_W. 
\end{align*}
See for a similar argument in \citet{wang2024penalized}.
By Theorem \ref{equ:thm-bound-of-feature-selection}, for all $k \in \Pi_j$, 
\begin{align*}
    \mathbb{E}\left[ \max_{k\in\Pi_j}\left\|(\bm W^{T+1}_j)^{k,:} - (\widehat{\bm W}^{T+1}_j)^{k,:}\right\|_2\right] \leq  \mathbb{E}\left[d\left({\boldsymbol{\theta}}^{T+1}, \Theta\right) \right] \lesssim \psi_{n,d}^{\frac{1}{2(a-1)}}. 
\end{align*}
This implies that 
\begin{align*}
   \mathrm{P}\left(\max_{k\in\Pi_j}\left\|(\bm W^{T+1}_j)^{k,:} - (\widehat{\bm W}^{T+1}_j)^{k,:}\right\|_2 \geq \psi_{n,d}^{\frac{1}{4(a-1)}}\right) \lesssim  \psi_{n,d}^{\frac{1}{4(a-1)}}. 
\end{align*}
Then, following the previous derivation, we have 
\begin{align*}
   \widehat{\tau} -  \left\|(\bm W^{T+1}_j)^{k,:} \right\|_2
  \leq  
 \left\| (\widehat{\bm W}^{T+1}_j)^{k,:}\right\|_2 -  \left\|(\bm W^{T+1}_j)^{k,:} \right\|_2 
   \leq   \left\|(\bm W^{T+1}_j)^{k,:} - (\widehat{\bm W}^{T+1}_j)^{k,:}\right\|_2.  
\end{align*}
Thus, we have 
\begin{align}\label{equ:important-variable-selection}
    \mathrm{P}\left( \min_{k \in \Pi_j} 
    \left\|(\bm W^{T+1}_j)^{k,:} \right\|_2  \geq \widehat{\tau} - \psi_{n,d}^{\frac{1}{4(a-1)}}  \right) \leq 1 - \psi_{n,d}^{\frac{1}{4(a-1)}}
\end{align}
by Markov's inequality. 
Turning to the unimportant variables, we use Theorem \ref{equ:thm-bound-of-feature-selection}.
For each $j\in [d]$, 
\begin{align*}
\mathbb{E}\left[ \left\|(\bm W^{T+1}_j)^{k,:} \right\|_2\right] \leq 
    \mathbb{E}\left[ \sum_{k \in \Pi_j^c }\left\| \left({\bm W}_j^{T+1}\right)^{k}\right\|_2\right] \lesssim  d \cdot \psi_{n,d}^{\frac{1}{2(a-1)}}, k\in [j-1]. 
\end{align*}
Thus, again by Markov's inequality, we have 
\begin{align*}
    \mathrm{P} \left( \max_{j}\left\|(\bm W^{T+1}_j)^{k,:} \right\|_2 \geq d^2 \cdot \psi_{n,d}^{\frac{1}{4(a-1)}}  \right) \leq & 
    \sum_{j=1}^d
    \mathrm{P} \left( \left\|(\bm W^{T+1}_j)^{k,:} \right\|_2 \geq d^2 \cdot \psi_{n,d}^{\frac{1}{4(a-1)}}  \right)\\
    \leq &d \cdot  d \cdot \psi_{n,d}^{\frac{1}{2(a-1)}}  / \left(d^2\cdot \psi_{n,d}^{\frac{1}{4(a-1)}}\right) = \psi_{n,d}^{\frac{1}{4(a-1)}}. 
\end{align*}
This, together with \eqref{equ:important-variable-selection} yields the derived conclusion. 
    
\end{proof}

\begin{lemma}[\textbf{Reduced Discriminator Complexity}]\label{lem:reduced-disco-complexity}
Remind that we defined $J_0^j = \{j\}$, $J_{k+1}^j = \bigcup_{j \in J_{k}^j} \Pi_j $, and $J_{\infty}^j = \lim_{k\to \infty} J_{k}^j$. 
For each $j$, there exists some $j'$ such that if $j\in J_{\infty}^k$, then $k \in J_{\infty}^{j'}$.
Then let $j' = h(j)$ from $[d]$ to $[d]$ be this map from $j$ to find its smallest ancestor feature. 
Let $\mathcal{A}$ be the collection of all such smallest ancestors. 
Define
\begin{align}
\mathcal{F}_{\Pi} = \left\{ \sum_{j\in \mathcal{A}} f_j (x^{J_{\infty}^j}) \;\;  \bigg| f_j \in \mathcal{F}_{Lip, 1} (\mathbb{R}^{|J_{\infty}^j|+1}) \right\}. 
\end{align}
For any data-generating process $\mathrm{P}$ and $\mathrm{Q}$ with Bayes networks, i.e., Assumption \ref{asp:bayesnetwork}, there holds 
\begin{align}
    \sup_{f \in \mathcal{F}_{Lip, 1} } \left( \mathbb{E}_{\mathrm{P}} \left[f(X)\right] - \mathbb{E}_{\mathrm{Q}} \left[f(X)\right] \right) 
    = \sup_{f \in \mathcal{F}_{\Pi}} 
    \left( \mathbb{E}_{\mathrm{P}} \left[f(X)\right] - \mathbb{E}_{\mathrm{Q}} \left[f(X)\right] \right).
\end{align}
\end{lemma}

\begin{proof}[\textbf{Proof of Lemma \ref{lem:reduced-disco-complexity}}]

We have 
\begin{align*}
    \sup_{f \in \mathcal{F}_{Lip, 1} } \left( \mathbb{E}_{\mathrm{P}} \left[f(X)\right] - \mathbb{E}_{\mathrm{Q}} \left[f(X)\right] \right)
    = \mathcal{W}(\mathrm{P}, \mathrm{Q})
    =  \sum_{j\in\mathcal{A}} \mathcal{W}(\mathrm{P}^{J^j_{\infty}}, \mathrm{Q}^{J^j_{\infty}}),
\end{align*}
where the last equation is due to the decomposition of the Wasserstein distance along independent features. 
Then, 
\begin{align*}
    \sum_{j\in\mathcal{A}} \mathcal{W}(\mathrm{P}^{J^j_{\infty}}, \mathrm{Q}^{J^j_{\infty}}) = & \sum_{j\in\mathcal{A}} 
    \sup_{f\in \mathcal{F}_{Lip, 1}(\mathbb{R}^{|J^j_{\infty}|+1})}(\mathbb{E}_{\mathrm{P}^{J^j_{\infty}}}\left[f(X)\right] 
    - \mathbb{E}_{\mathrm{Q}^{J^j_{\infty}}}\left[f(X)\right])\\
    = & \sup_{f \in \mathcal{F}_{\Pi} } \left( \mathbb{E}_{\mathrm{P}} \left[f(X)\right] - \mathbb{E}_{\mathrm{Q}} \left[f(X)\right] \right).
\end{align*}
    
\end{proof}

\begin{proof}[\textbf{Proof of Theorem \ref{thm:estimation}}]

Lemma \ref{lem:exact-selection} indicates that $\left\|(\bm W^{T+1}_j)^{k,:} \right\|_2 = 0$ if and only if $k\notin \Pi_j$. 
Then by definition, with probability $1- 2 \psi_{n,d}^{\frac{1}{4(a-1)}}$, there holds
\begin{align*}
     \mathbb{E}\left[ \mathcal{W} \left( \mathrm{P}_{g_{\overline{\bm \theta}^{T+1}}(Z)}, \mathrm{P}_X\right)\right] = \mathbb{E} \left[\sup_{f \in \mathcal{F}_{Lip,1}} 
     \left(
     R(\overline{\bm \theta}^{T+1}, f)
     \right)
     \right] = \mathbb{E} \left[\sup_{f \in \mathcal{F}_{\Pi}} 
     \left(
     R(\overline{\bm \theta}^{T+1}, f)
     \right)
     \right]. 
\end{align*}
Following the same strategy of \eqref{equ:proof-uniform-converge-on-l1-function-class}, \eqref{equ:proof-derive-approximiation-error}, and \eqref{equ:proof-use-of-assumption-optimizaiton}, we get 
\begin{align*}
    \mathbb{E} \left[\sup_{f \in \mathcal{F}_{\Pi}} 
     \left(
     R(\overline{\bm \theta}^{T+1}, f)
     \right)
     \right] \lesssim (\sqrt{d} + \log n) n^{- \frac{1}{s}} + \sup_{f \in \mathcal{F}_{Lip, 1}}  \inf_{f' \in \mathcal{F}} \left\|f - f'\right\|_{\infty} + \overline{\mathcal{E}}. 
\end{align*}
Here, the exponential component on the first term becomes $-1 /s$, where $s = \max_{j\in[d]}J_{\infty}^j$. 
This is because Lemma 1 in \citet{wang2024penalized} used the bound of the covering number 
\begin{align*}
    \log N (\delta, \mathcal{F}_{Lip,1}, \|\cdot\|_{\infty}) \leq C \left(\frac{\log n}{\delta}\right)^{d},
\end{align*}
while here, we can consider a smaller bound 
\begin{align*}
    \log N (\delta, \mathcal{F}_{\Pi}, \|\cdot\|_{\infty}) \leq & \log \left(d \cdot N (\delta, \mathcal{F}_{Lip,1}(\mathbb{R}^s), \|\cdot\|_{\infty})\right)\\
    \leq & \log d + C \left(\frac{\log n}{\delta}\right)^{s}.  
\end{align*}

\end{proof}

\section{More on Experimental Settings}\label{exp_settings}

\subsection{Dataset Information}
We make a summary of these datasets’ basic statistics in Table \ref{tab:realdatainfo}.
These datasets are widely adopted in the literature of tabular data \citep{cai2023extrapolated}.
In this table, we represent the number of records, number of attributes, and attributes’
domain size range of these datasets. All of these datasets have been processed to ensure that there are no missing values.

\begin{table}[htbp]
\centering
\caption{Summary of real datasets.}
\label{tab:realdatainfo}
\resizebox{0.6\linewidth}{!}{
\renewcommand{\arraystretch}{1}
\setlength{\tabcolsep}{5pt}
\begin{tabular}{r|r|r|r|r}
\toprule
\multicolumn{1}{c|}{Dataset } & \multicolumn{1}{c|}{$n$} & \multicolumn{1}{c|}{ Attr } &  \multicolumn{1}{c|}{\begin{tabular}[c]{@{}c@{}}Min/Max\\ Domain Range\end{tabular}} & \multicolumn{1}{c}{\begin{tabular}[c]{@{}c@{}}Min/Max\\ Domain Range after\\ preprocessing\end{tabular}}  \\ \midrule
\textsc{California}         & 20640                                 & 8                                               & 1.652 $\sim$ 35679.000                                                                               & 4.378 $\sim$ 92.931                                                                                 \\ 
\textsc{House-16H}          & 22784                                 & 17                                     &  0.512 $\sim$ 7322562.000                                                                                & 3.027  $\sim$ 84.820                                                                         \\ 
\textsc{Cpu-act}          & 8192                                & 20                        & 20.120 $\sim$ 2526371.000                                                                               & 4.761 $\sim$  36.358                                                                               \\ \bottomrule
\end{tabular}
}

\end{table}

\subsection{Computational Resources}

All experiments are conducted on a machine with 72-core Intel Xeon 2.60GHz and 128GB of main memory.

\subsection{DPSGD}
A gradient-based training algorithm can be privatized by using differentially private stochastic gradient descent (DPSGD) \citep{abadi2016deep, bassily2014private, song2013stochastic} as a direct replacement for standard SGD. DPSGD operates by clipping per-example gradients and adding Gaussian noise to their aggregated sum, thereby bounding and obfuscating the influence of any single data point on the learned model parameters. The privacy guarantees of DPSGD are derived using classical tools from the DP literature, including the Gaussian mechanism, privacy amplification via subsampling, and composition theorems \citep{abadi2016deep, balle2018privacy, dwork2006our, wang2019subsampled}. In our work, we implemented the DPSGD analysis in the Opacus library \citep{yousefpour2021opacus}.

\subsection{Hyperparameters}

We summarize the hyperparameter configurations used in our experiments in Section~\ref{real_data}. Table~\ref{tab:fixed_params} reports the set of hyperparameters that remain fixed across all privacy levels ($\varepsilon$) for each dataset. These include learning rates for the generator and discriminator, regularization parameters ($\lambda$, $\gamma$), the generator-to-discriminator update ratio, and the batch size. Table~\ref{tab:epsilon_params} presents the privacy-dependent hyperparameters, including the noise scale $\sigma$ used for DP and the number of discriminator update steps. These values are tuned individually for each dataset and $\varepsilon$ to balance the trade-off between privacy and utility.

\begin{table}[h]
\centering
\caption{Hyperparameters choices for \texttt{PrAda-GAN} across all $\varepsilon$ values.}
\label{tab:fixed_params}
\begin{tabular}{l|ccc}
\toprule
\textbf{Hyperparameter} & \textsc{California}& \textsc{House-16H} & \textsc{Cpu-act}\\
\midrule
$d_{lr}$ & 0.01 & 0.01 & 0.01  \\
$g_{lr}$& 0.001 & 0.001 & 0.001 \\
$\lambda$& 0.003 & 0.003 &  0.003\\
$\gamma$& 0.0 & 0.0 & 0.2 \\
$t_g$ & 1:10 & 1:10 & 1:10 \\
Batch Size & 50 & 50 & 50 \\
\bottomrule
\end{tabular}
\end{table}

\begin{table}[h]
\centering
\caption{Privacy-dependent hyperparameters for each dataset with clipping threshold $C=1.0$. }
\label{tab:epsilon_params}
\begin{tabular}{l|ccc|ccc|ccc}
\toprule
\textbf{Hyperparameter} & \multicolumn{3}{c|}{\textsc{California}} & \multicolumn{3}{c|}{\textsc{House-16H}} & \multicolumn{3}{c}{\textsc{Cpu-act}} \\
$\varepsilon$ & $0.2$ & $1$ & $5$ &  $0.2$ & $1$ & $5$ &  $0.2$ & $1$ & $5$ \\
\midrule
$\sigma$ & 10.00 & 2.00 & 0.80 & 7.90 & 2.00 & 0.77 & 13.00 & 4.00 & 1.20 \\
Discriminator Steps & 8000 & 7000 & 8000 & 6000 & 9000 & 8000 & 4000 & 7000 & 8000 \\
\bottomrule
\end{tabular}
\end{table}

\subsection{Model Architecture}
\label{sec:model_architecture}
Our GAN framework comprises a sequential generator and a unified discriminator. The generator is composed of multiple one-hidden-layer residual networks, each tasked with generating a single attribute conditioned on all previously generated attributes. These generators are trained in a fixed topological order that reflects an estimated Bayesian network structure. Each subgenerator features a hidden layer of size 10 with LeakyReLU activation. The discriminator is a fully connected network with two hidden layers, also using LeakyReLU activations. The first layer has a width equal to the input dimension (i.e., the number of attributes), while the second layer has half that size. Architectural configurations for each dataset are detailed in Table~\ref{tab:arch}.

\begin{table}[p]
\centering
\caption{Generator and discriminator architectures for each dataset. The discriminator consists of two layers: the first layer has the same dimensionality as the input data, while the second layer has half that width. }
\label{tab:arch}
\begin{tabular}{l|cc|cc}
\toprule
\textbf{Dataset} & \multicolumn{2}{c|}{\textbf{Generator}} & \multicolumn{2}{c}{\textbf{Discriminator}} \\
 & Num Generators & Hidden Layers & Input Dim & Hidden Layers \\
\midrule
\textsc{California} & 8  & [10]       & 8  & [8, 4] \\
\textsc{House-16H}  & 17 & [10]       & 17 & [17, 8] \\
\textsc{Cpu-act}    & 20 & [10]       & 20 & [20, 10] \\
\bottomrule
\end{tabular}
\end{table}

\subsection{Evaluation Metrics}\label{sec:evaluation}

We define the TVD as $ \frac{1}{2} \sum_{1\le i\le j\le d} |M_{i,j}^{syn} - M_{i,j}^{test}|$ where $M_{i,j}^{syn}$ and $M_{i,j}^{test}$ represent the 2-way marginals determined by the synthetic dataset and the test dataset respectively. 
See \citet{chen2025benchmarking} for details. 
Additional metrics, including Maximum Mean Discrepancy (MMD), Jensen-Shannon Divergence (JS), and 1-way marginal discrepancy (TVD (1-way)), are provided for comprehensive comparison.

\section{Additional Experimental Results}\label{exp_results}
\subsection{Synthetic Experiments}

Under the nonlinear setting described in Section~\ref{nonlinear_exp}, additional results for the analysis of learning rates and penalty parameters—evaluated using alternative metrics such as MMD, JS divergence, and one-way marginals—are provided in Figures~\ref{fig:ER_nonlinear_g_lr_d_lr_2} and~\ref{fig:ER_nonlinear_gamma_lambda_2}, respectively.

Under linear setting, the ground truth Bayes networks are generated from the scale-free (SF) graph model, with a fixed number of nodes $d=10$. For each generated network, we simulated $n=2000$ training samples. The synthetic data are generated according to a structural equation model (SEM) of the form: $X_j = f_j(\Pi_j) + z_j$, $j \in [d]$, where $z_j \sim \mathcal{N}(0,1)$, are independent noise terms and recall that $\Pi_j$ denotes the parent set of $X_j$ in the underlying Bayes network. 
Each function $f_j$ is specified as a linear combination of its parent variables: $f_j(\Pi_j ) = \sum_{j'\in \Pi_j} w_{j,j'} X^{j'}$.
Each weight $w_{j,j'}$ is sampled independently from a uniform distribution over $[0.5, 2.0]$ and is independently negated with probability $0.5$. 
The results for the parameter analysis of learning rates under the linear $f_j(\Pi_j)$ setting are shown in Figures~\ref{fig:SF_linear_g_lr_d_lr_1} and~\ref{fig:SF_linear_g_lr_d_lr_2}, while Figures~\ref{fig:SF_linear_lambda_gamma} and~\ref{fig:SF_linear_lambda_gamma_2} present the performance across varying penalty parameters $\lambda$ and $\gamma$.

\begin{figure}[p]
    \centering
    \begin{minipage}{\columnwidth}
        \centering
        \includegraphics[width=0.8\linewidth]{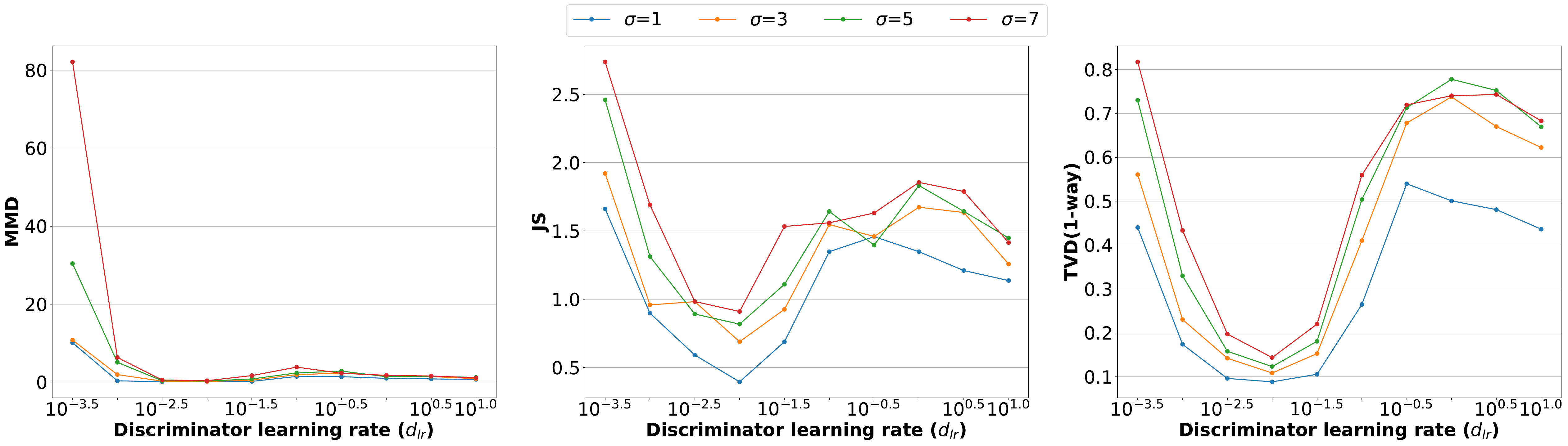}
        \label{fig:SF_d_lr_MMD}
    \end{minipage}
    \\  
    \begin{minipage}{\columnwidth}
        \centering
        \includegraphics[width=0.8\linewidth]{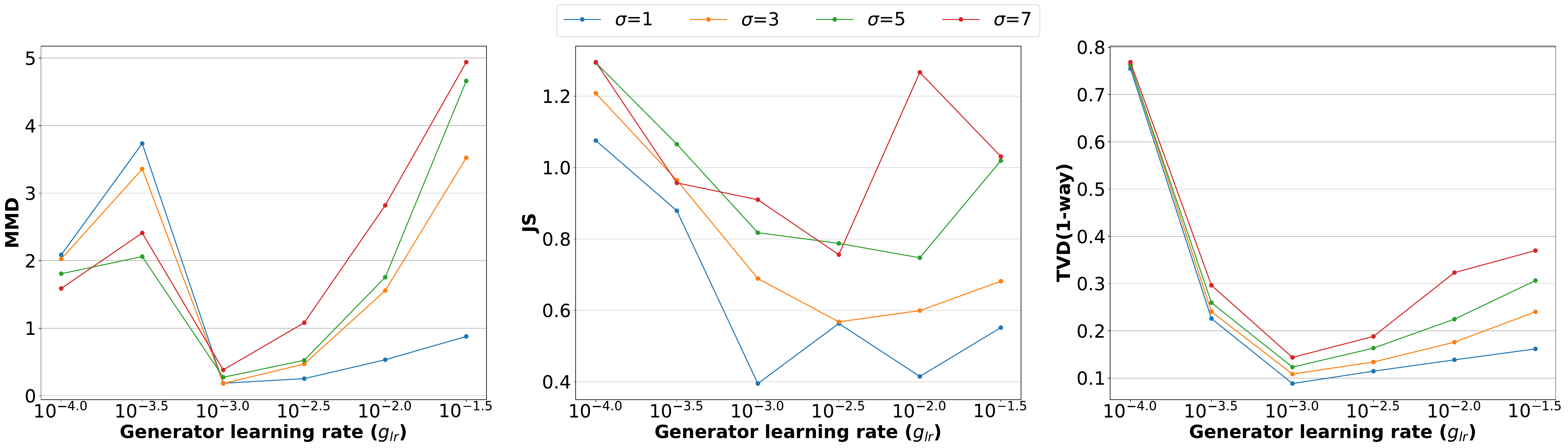}
        \label{fig:SF_g_lr_MMD}
    \end{minipage}
    \caption{Average MMD, JS and TVD(1-way) under nonlinear function $f_j(\Pi_j )$. Top: varying discriminator learning rates \( d_{\text{lr}} \) across $10^{h}$ for $h \in \{-3.5, \dots, 1\}$ with \( g_{\text{lr}} = 10^{-3} \). Bottom: varying generator learning rates \( g_{\text{lr}} \) across $10^{h}$ for $h \in \{-4, -3.5, \dots, -1.5\}$ with \( d_{\text{lr}} = 10^{-1} \).
}
    \label{fig:ER_nonlinear_g_lr_d_lr_2}
\end{figure}

\begin{figure}[p]
    \centering
    \begin{minipage}{\columnwidth}
        \centering
        \includegraphics[width=0.8\linewidth]{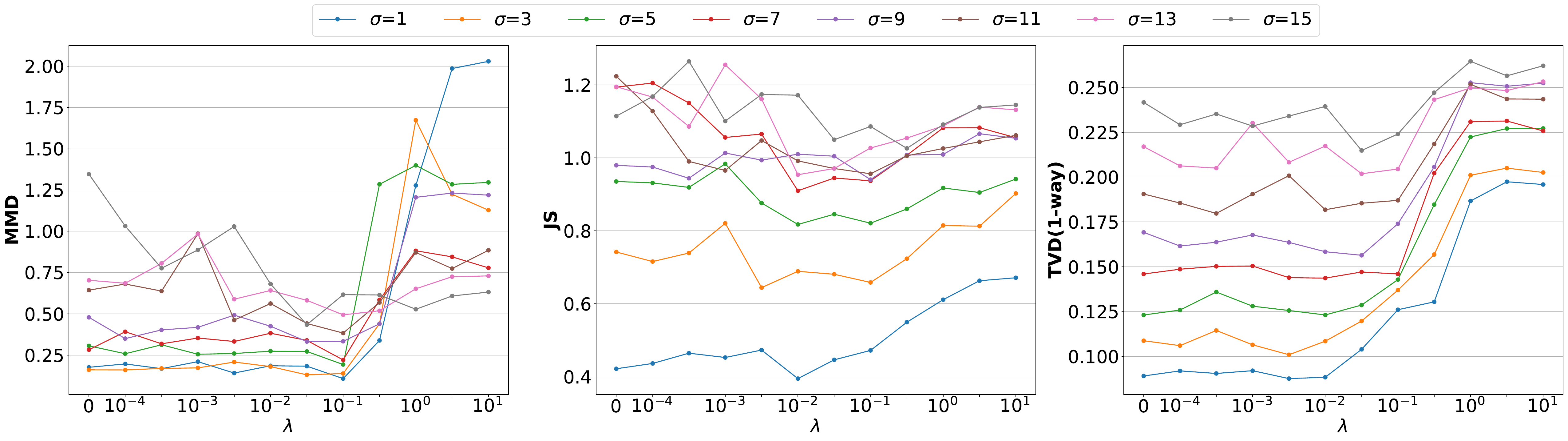}
        \label{fig:SF_lambda_MMD}
    \end{minipage}
    \\  
    \begin{minipage}{\columnwidth}
        \centering
        \includegraphics[width=0.8\linewidth]{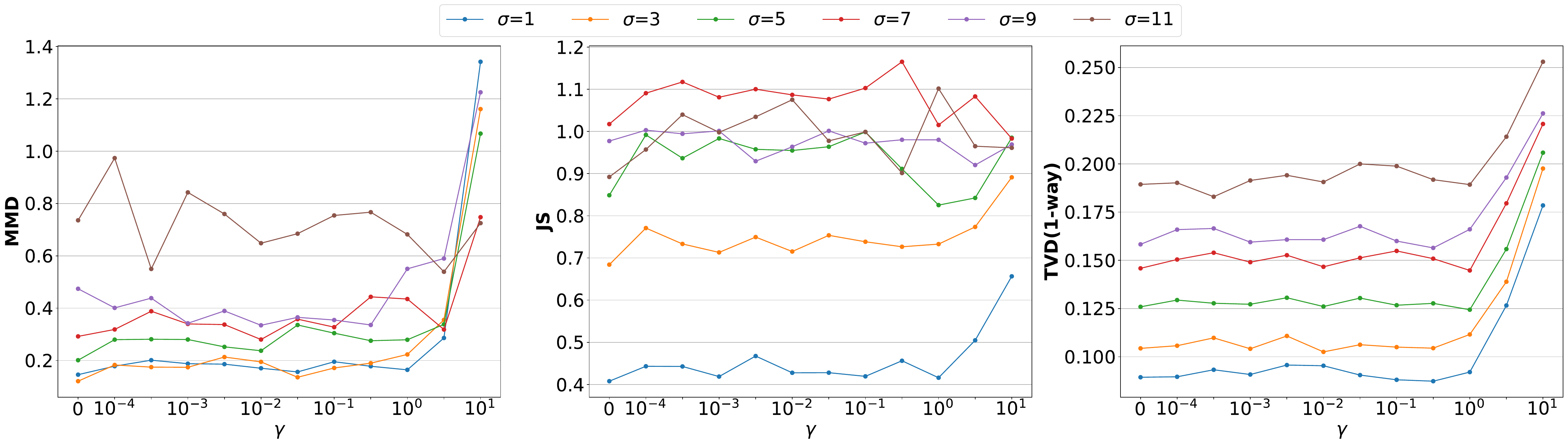}
        \label{fig:SF_gamma_MMD}
    \end{minipage}
    \caption{Average MMD, JS and TVD(1-way) under nonlinear function $f_j(\Pi_j )$, evaluated across varying $\lambda$ and $\gamma$.}.
    \label{fig:ER_nonlinear_gamma_lambda_2}
\end{figure}

\begin{figure}[p]
    \centering
    \begin{minipage}{\columnwidth}
        \centering
        \includegraphics[width=0.8\linewidth]{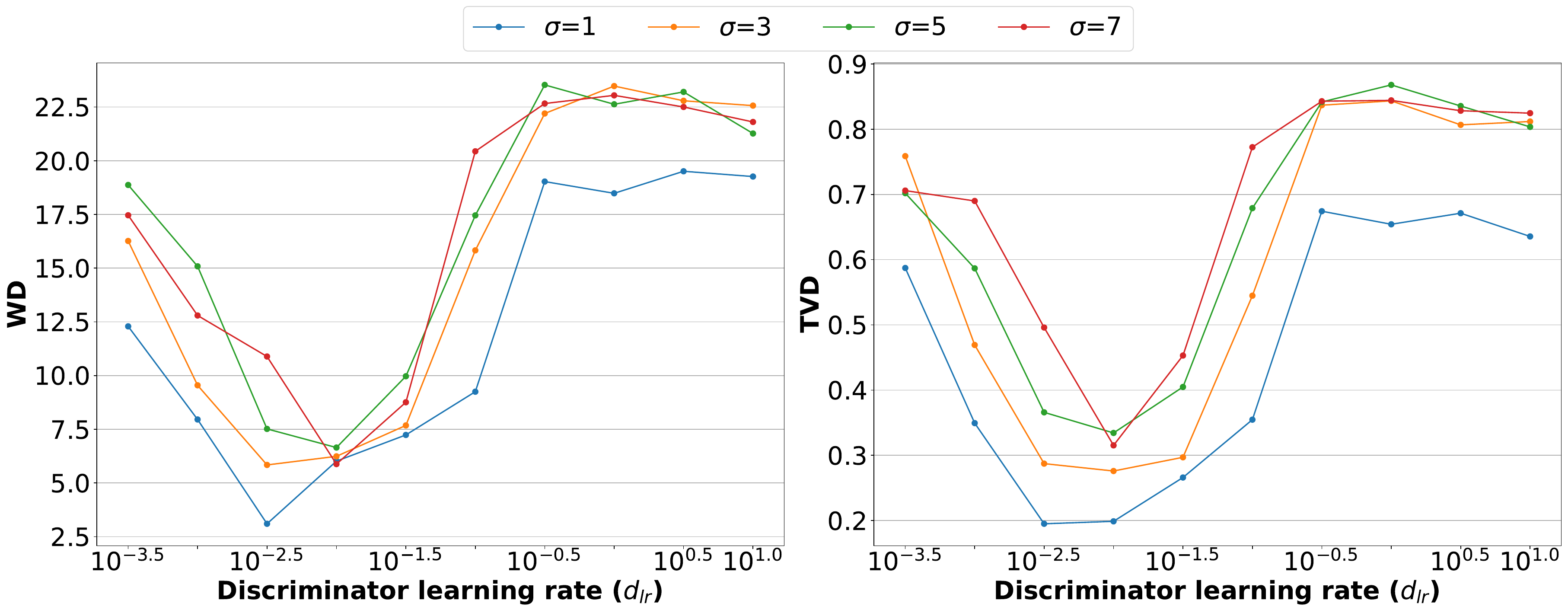}
        \label{fig:SF_d_lr_1}
    \end{minipage}
    \\  
    \begin{minipage}{\columnwidth}
        \centering
        \includegraphics[width=0.8\linewidth]{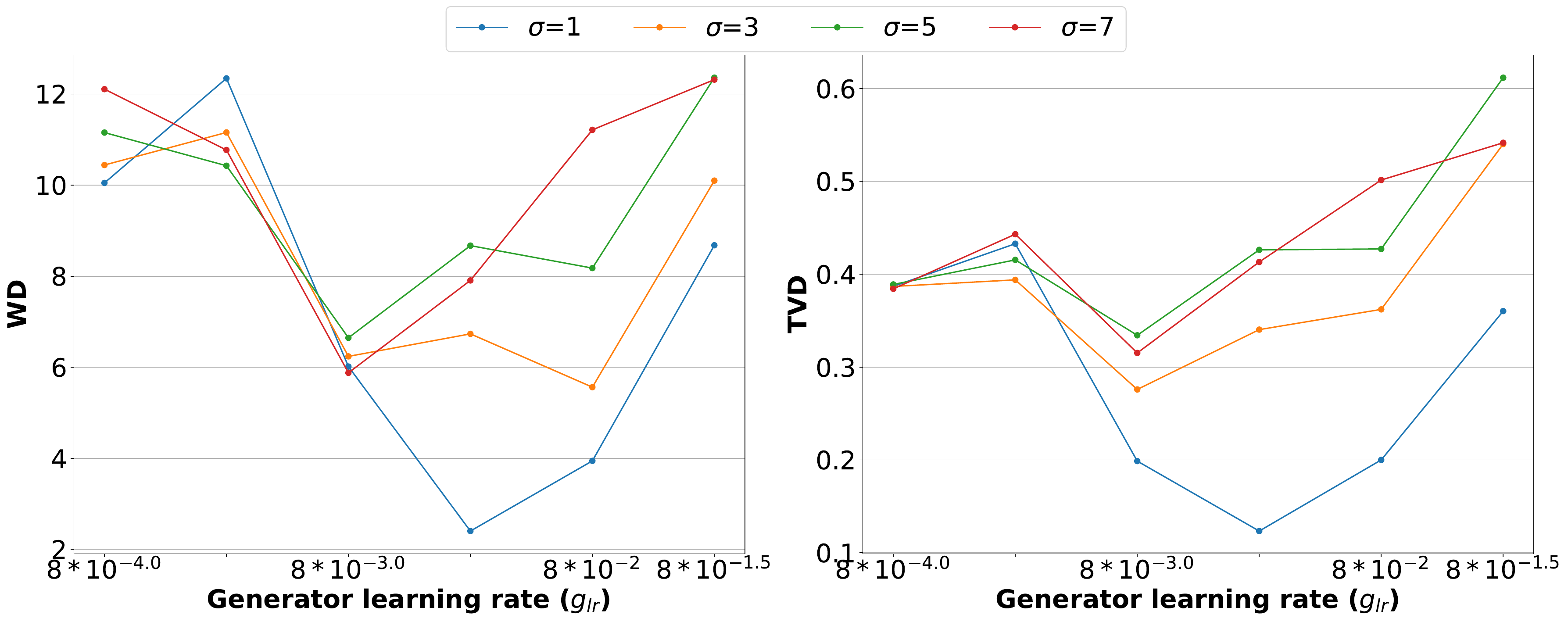}
        \label{fig:SF_g_lr_1}
    \end{minipage}
    \caption{Average WD and TVD. Top: varying discriminator learning rates \( d_{\text{lr}} \) across $10^{h}$ for $h \in \{-3.5, \dots, 1\}$ with \( g_{\text{lr}} = 8*10^{-3} \). Bottom: varying generator learning rates \( g_{\text{lr}} \) across $8*10^{h}$ for $h \in \{-4, -3.5, \dots, -1.5\}$ with \( d_{\text{lr}} = 10^{-1} \).
}
    \label{fig:SF_linear_g_lr_d_lr_1}
\end{figure}

\begin{figure}[p]
    \centering
    \begin{minipage}{\columnwidth}
        \centering
        \includegraphics[width=0.8\linewidth]{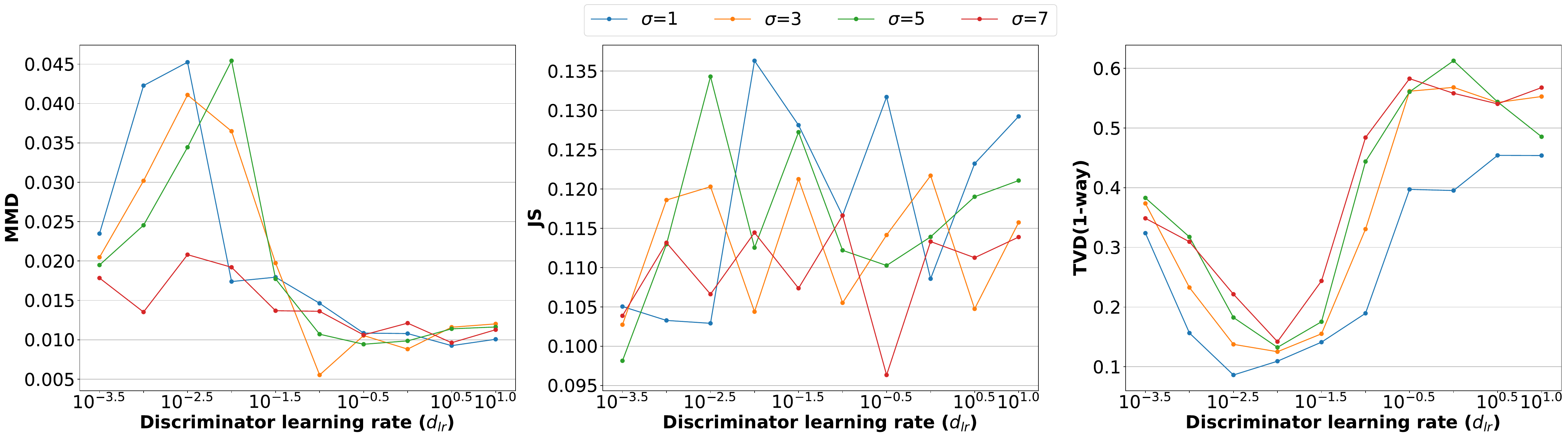}
        \label{fig:SF_d_lr_2}
    \end{minipage}
    \\  
    \begin{minipage}{\columnwidth}
        \centering
        \includegraphics[width=0.8\linewidth]{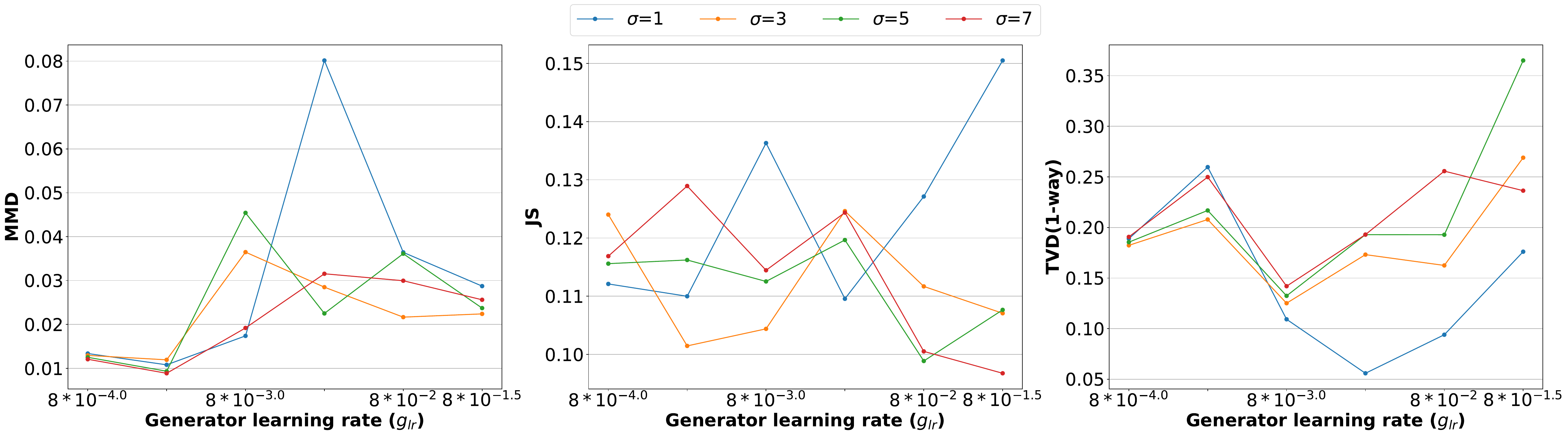}
        \label{fig:SF_g_lr_2}
    \end{minipage}
    \caption{Average MMD, JS and TVD(1-way). Top: varying discriminator learning rates \( d_{\text{lr}} \) across $10^{h}$ for $h \in \{-3.5, \dots, 1\}$ with \( g_{\text{lr}} = 8*10^{-3} \). Bottom: varying generator learning rates \( g_{\text{lr}} \) across $8*10^{h}$ for $h \in \{-4, -3.5, \dots, -1.5\}$ with \( d_{\text{lr}} = 10^{-1} \).
}
    \label{fig:SF_linear_g_lr_d_lr_2}
\end{figure}

\begin{figure}[p]
    \centering
    \begin{minipage}{\columnwidth}
        \centering
        \includegraphics[width=0.8\linewidth]{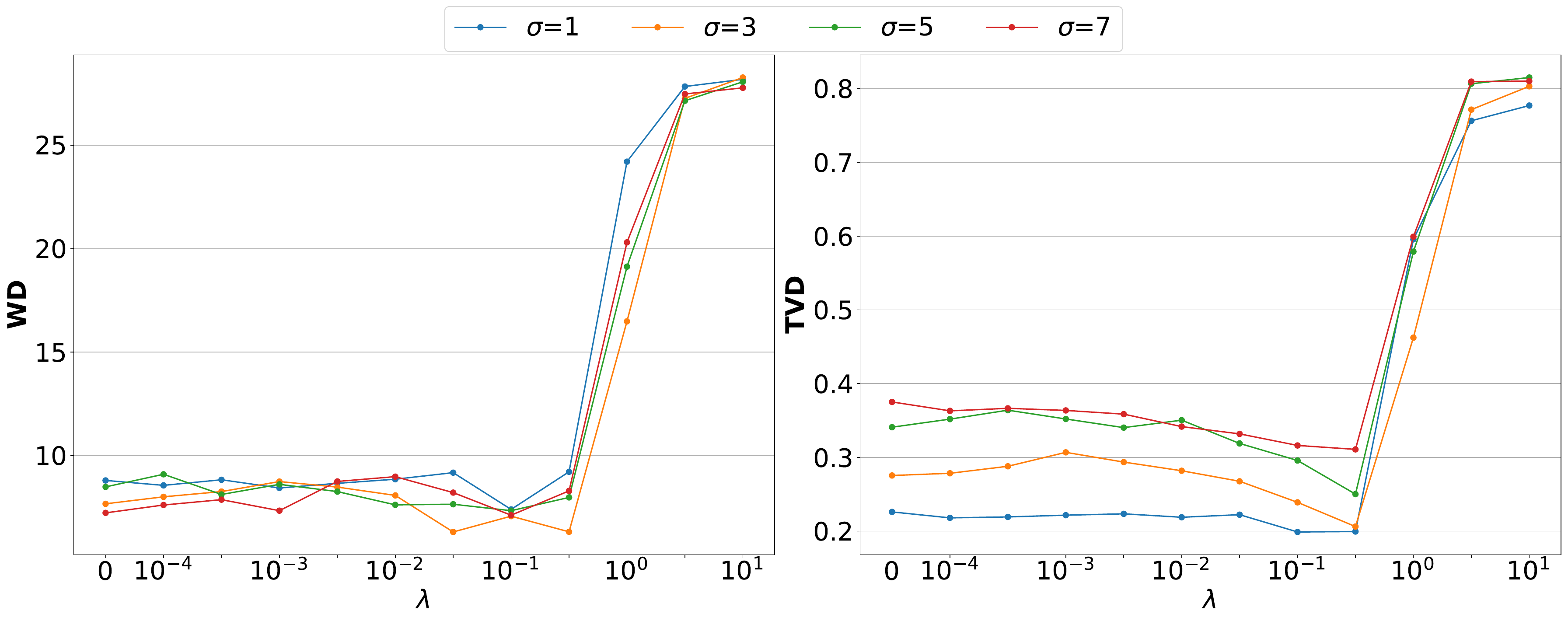}
        \label{fig:SF_lambda_1}
    \end{minipage}
    \\  
    \begin{minipage}{\columnwidth}
        \centering
\includegraphics[width=0.8\linewidth]{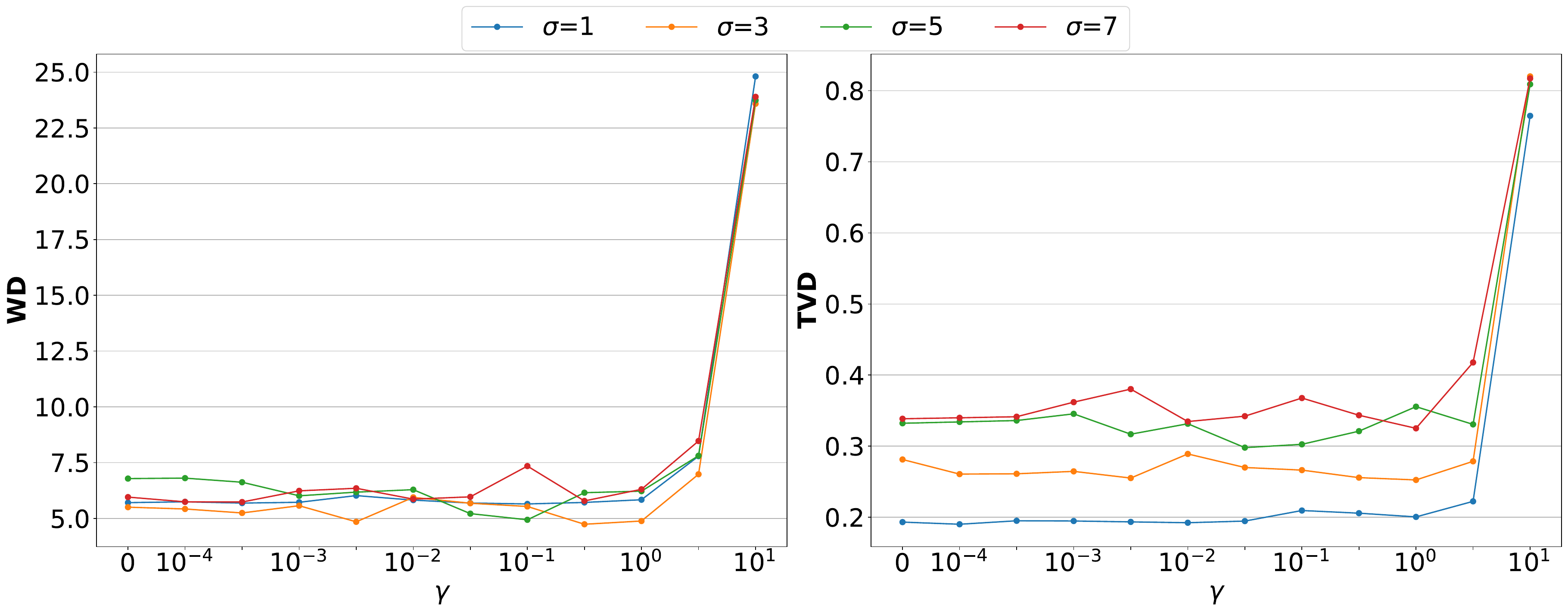}
        \label{fig:SF_gamma_1}
    \end{minipage}
    \caption{Average WD and TVD under varying $\lambda$ and $\gamma$.}
    \label{fig:SF_linear_lambda_gamma}
\end{figure}

\begin{figure}[p]
    \centering
    \begin{minipage}{\columnwidth}
        \centering
        \includegraphics[width=0.8\linewidth]{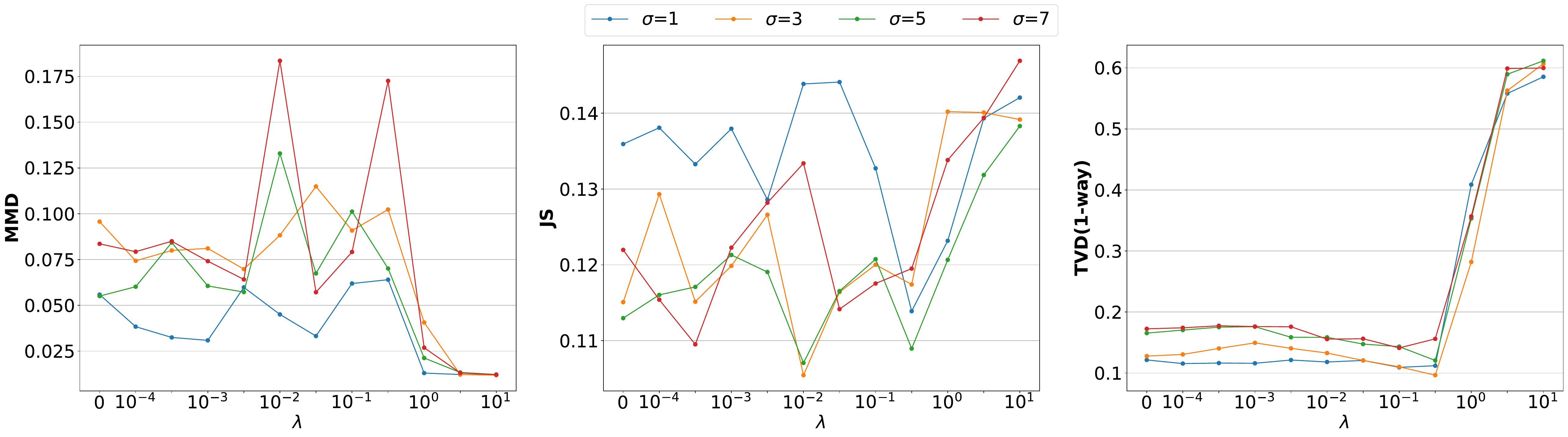}
        \label{fig:SF_lambda_2}
    \end{minipage}
    \\  
    \begin{minipage}{\columnwidth}
        \centering
        \includegraphics[width=0.8\linewidth]{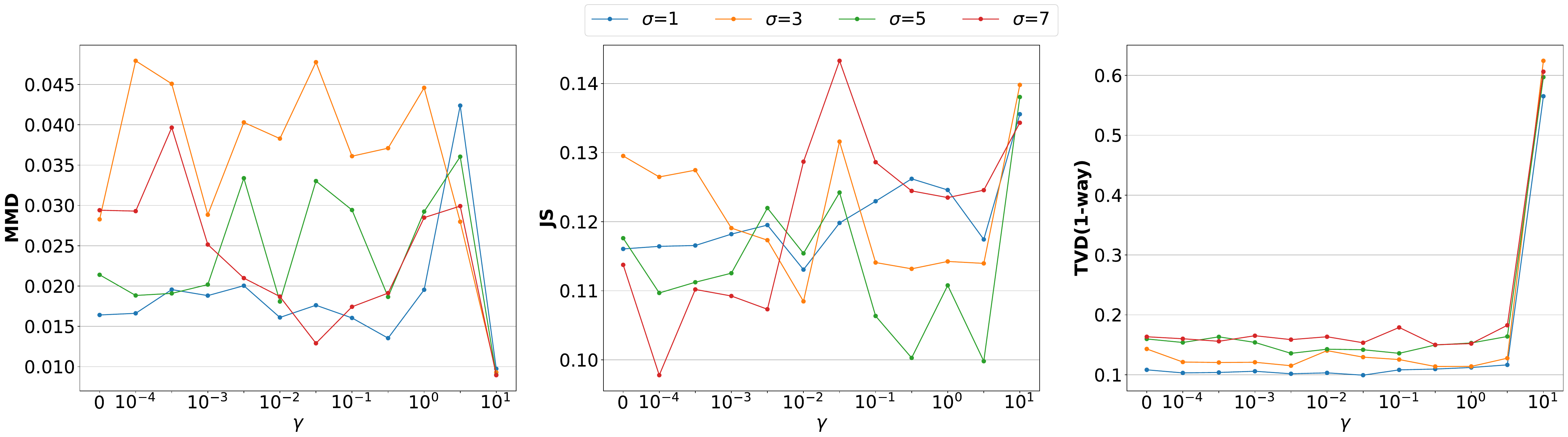}
        \label{fig:SF_gamma_2}
    \end{minipage}
    \vskip -0.1in
    \caption{Average MMD, JS and TVD(1-way) under varying $\lambda$ and $\gamma$.
}\label{fig:SF_linear_lambda_gamma_2}
\end{figure}

\subsection{Real Data Comparison}
We provide additional results to complement Section~\ref{real_data}. Table~\ref{tab:mmd_js_1-way} reports distributional similarity metrics, including MMD, JS divergence, and one-way marginals, while Table~\ref{tab:rmse} presents RMSE results evaluating downstream machine learning performance.

\begin{table*}[p]
\centering
\scriptsize
\caption{
Real data comparison for distribution similarity MMD, JS, and 1-way marginal.
Best results are in \textbf{bold}; the second-best are \underline{underlined}.
The best results that hold significance over the others have a *.
}
\label{tab:mmd_js_1-way}
\resizebox{1\linewidth}{!}{
\renewcommand{\arraystretch}{1.2}
\setlength{\tabcolsep}{2pt}
\begin{tabular}{lccccccccc cccccccccc cccccccccc}
\toprule
\textbf{Dataset} & \multicolumn{9}{c}{\textbf{\textsc{California}}} & \multicolumn{9}{c}{\textsc{House-16H}} & \multicolumn{9}{c}{\textsc{Cpu-act}} \\
\cmidrule(lr){2-10} \cmidrule(lr){11-19} \cmidrule(lr){20-28}
$\varepsilon$ & \multicolumn{3}{c}{0.2} & \multicolumn{3}{c}{1} & \multicolumn{3}{c}{5} 
& \multicolumn{3}{c}{0.2} & \multicolumn{3}{c}{1} & \multicolumn{3}{c}{5} 
& \multicolumn{3}{c}{0.2} & \multicolumn{3}{c}{1} & \multicolumn{3}{c}{5} \\
Metric & MMD & JS & 1-way & MMD & JS & 1-way  & MMD & JS & 1-way 
& MMD & JS & 1-way  & MMD & JS & 1-way  & MMD & JS & 1-way 
& MMD & JS & 1-way  &MMD & JS & 1-way  &MMD & JS & 1-way  \\
\midrule
\texttt{PrAda-GAN} & \textbf{0.213}* &  \underline{2.779}* &  \textbf{0.203}* &  \textbf{0.242}* &  \underline{2.660} &  \textbf{0.159}* &  \textbf{0.077}* & \underline{2.640} &  \textbf{0.142}* &  \textbf{0.425}* & \underline{4.317} & 0.222 &  \textbf{0.275}* &  \underline{3.624} & 0.154 &  \textbf{0.427}* &  \underline{3.536} &  \underline{0.126} &  \textbf{2.341}* & \underline{4.338} & 0.317 &  \textbf{0.783}* &  \textbf{2.766}* &  \textbf{0.190}* & 0.990 &  \textbf{2.311}* &  \textbf{0.157}* \\
\texttt{AIM} &2.274 & 3.330 & \underline{0.274} &1.673 & 3.132 & 0.252 & \underline{1.390} & 2.895 & 0.226 & 4.770 & 5.064 & \underline{0.191} & 2.270 & 4.773 &  \underline{0.149} &  \underline{1.881} & 4.618 & 0.148 & 12.746 & 5.972 &  \textbf{0.304} & 1.309 &  5.227 &  \underline{0.230} &  \underline{0.647} &  \underline{4.895} & \underline{0.195} \\
\texttt{PrivMRF}& \underline{2.010} & 3.235 & 0.275 & \underline{1.491} & 2.850 & \underline{0.207} & 1.903 & 2.860 & \underline{0.201} & \underline{3.430} & 5.160 &  \textbf{0.189} & \underline{1.991} & 4.630 &  \textbf{0.130}* & 2.042 & 4.439 &  \textbf{0.097}* &  \underline{6.552} &  5.676 &  \underline{0.305} &  \underline{0.798} & 5.150 & 0.234 & \textbf{0.451} & 4.967 & 0.199  \\
\texttt{GEM} & 9.680 & 3.291 & 0.379 & 2.011 & 3.366 & 0.271 & 1.628 & 3.189 & 0.268 & 59.783 & 5.718 & 0.391 & 8.865 & 5.470 & 0.257 & 9.294 & 5.489 & 0.262 & 77.804 & 6.536 & 0.451 & 15.533 & 5.472 & 0.285 & 9.309 & 5.204 & 0.232 \\
\texttt{DP-MERF}& 4.366 & 3.672 & 0.480 & 4.721 & 3.783 & 0.433 & 2.202 & 3.865 & 0.433 & 5.576 & 6.005 & 0.380 & 2.611 & 5.784 & 0.317 & 3.086 & 5.837 & 0.311 & 13.269 & 7.871 & 0.448 & 4.223 & 7.215 & 0.427 & 4.314 & 7.330 & 0.413  \\
\texttt{PrivBayes}& 576.149 & \textbf{1.942} & 0.319 & 427.872& \textbf{1.970}* & 0.313 & 401.090 &  \textbf{1.874}* & 0.311 & 8.778 &  \textbf{3.102}* & 0.246 & 3.411 &  \textbf{2.563}* & 0.215 & 3.084 &  \textbf{2.379}* & 0.210 & 195.08 & \textbf{4.145}* & 0.350 & 34.221 & \underline{3.283} & 0.279 & 10.462 & 3.087 & 0.413 \\
\midrule
\textbf{Ground Truth} & 0.001 & 0.202 & 0.015 & 0.001 & 0.202 & 0.015  & 0.001 & 0.202 & 0.015 & 0.004 & 0.311 & 0.020 &0.004 & 0.311 & 0.020& 0.004 & 0.311 & 0.020 & 0.003& 0.608& 0.066 &0.003& 0.608& 0.066 & 0.003& 0.608& 0.066 \\
\bottomrule
\end{tabular}%
}
\end{table*}

\begin{table}[p]
\centering
\tiny
\caption{Real data comparison for downstream machine learning efficacy (RMSE).
Best results are in \textbf{bold}; second-best are \underline{underlined}.
The best results that hold significance over the others have a *.}
\label{tab:rmse}
\resizebox{1\linewidth}{!}{
\renewcommand{\arraystretch}{0.99}
\setlength{\tabcolsep}{3.5pt}
\begin{tabular}{@{}lccccccccccccccc@{}}
\toprule
& \multicolumn{5}{c}{\textsc{California} } & \multicolumn{5}{c}{\textsc{House-16H}} & \multicolumn{5}{c}{\textsc{Cpu-act}} \\
\cmidrule(lr){2-6} \cmidrule(lr){7-11} \cmidrule(lr){12-16}
\textbf{Method} & Cat & MLP & RF & XGB & SVM & Cat & MLP & RF & XGB & SVM & Cat & MLP & RF & XGB & SVM \\
\midrule
\multicolumn{16}{c}{\textbf{$\varepsilon=0.2$}} \\
\texttt{PrAda-GAN}      & \textbf{0.829}* & \textbf{0.867}* & \textbf{0.844}* & \textbf{0.851}* & \textbf{0.879}* & \textbf{0.961}* & \textbf{0.989}* & \textbf{0.990} & \textbf{0.992}* & \textbf{0.980} & \textbf{0.939}* & \textbf{1.012} & \textbf{0.958}* & \textbf{0.959}* & \underline{1.000}\\
\texttt{AIM}       & 0.993 & 1.000 & 1.091 & 1.007 & 1.001 & 1.055 & 1.059 & 1.071 & 1.125 & \underline{1.037} & \underline{1.007} & 1.088 & 0.989 & \underline{1.027} & 1.047\\
\texttt{PrivMRF}   & 0.999 & \underline{0.995} & 1.113& 1.010&1.008 & 1.041 &  \underline{1.041} & \underline{1.047} & \underline{1.070} & 1.038 & 1.009 &1.108  & \underline{0.986} & 1.028 & 1.038 \\
\texttt{GEM}       & 1.067 & 1.064& 1.204& 1.106& 1.109 & 1.180 & 1.316 & 1.336 & 1.552 & 1.178  & 1.155 & 1.211 & 1.381 & 1.221 & 1.178 \\
\texttt{DP-MERF}   &\underline{0.890} & 1.064& \underline{1.026} & \underline{0.925} & \underline{0.969} &  \underline{1.001} & 1.300 & 1.484 & 1.396 & 1.097 & 0.951  & \underline{1.039}  & 1.004 & 0.986 &\textbf{0.939}* \\
\texttt{PrivBayes} &0.988 & 1.020& 1.043 & 1.048& 1.015  &1.040  & 1.183 & 1.078  & 1.149  & 1.050 & 1.078 &1.386  & 1.956 & 1.357 & 1.011 \\
\midrule
\multicolumn{16}{c}{\textbf{$\varepsilon=1.0$}} \\
\texttt{PrAda-GAN}      &\textbf{0.711}* &\textbf{0.720}* & \textbf{0.750}* & \textbf{0.746}* &\textbf{0.754}* &  \textbf{0.903}* & \textbf{0.935}* & \textbf{0.928}* & \textbf{0.929}* & \textbf{0.949}* & \underline{0.942} & \textbf{0.970}* & \textbf{0.955} & \underline{0.954} & \underline{0.959}\\
\texttt{AIM}       &0.998 & \underline{0.992}&1.102 &1.019 & 1.005& \underline{0.951} & \underline{0.970} & \underline{0.995} & \underline{1.059} & \underline{0.958}  & 0.987 & 1.064 & \underline{0.957} & 1.002& 1.036\\
\texttt{PrivMRF}   & 0.999& 1.005&1.102 & 1.012 &1.007 & 1.036 & 1.038 & 1.051& 1.066  & 1.037  & 0.997 & 1.099 & 0.970 & 1.010 & 1.044\\
\texttt{GEM}       & 1.021& 1.039 & 1.160 & 1.043& 1.035& 1.079 & 1.116  & 1.140 & 1.224 & 1.057 &  0.954& 1.080 & 1.057& 0.993 & 0.988 \\
\texttt{DP-MERF}   & \underline{0.840} & 1.086 & \underline{0.993} & \underline{0.858} & \underline{0.911} & 0.968 & 1.251 & 1.247 & 1.221 & 1.002  & \textbf{0.930}* & \underline{1.008} & 0.964& \textbf{0.950} & \textbf{0.957}\\
\texttt{PrivBayes} & 0.991 & 1.009 & 1.037 &1.043  & 1.003 & 1.034 & 1.238 & 1.038 & 1.093 & 1.046 & 1.002& 1.268& 2.121 & 1.228 & 1.025\\
\midrule
\multicolumn{16}{c}{\textbf{$\varepsilon=5.0$}} \\
\texttt{PrAda-GAN}      & \underline{0.621} & \textbf{0.627}*& \textbf{0.651}* & \underline{0.649} & \textbf{0.643}* & 0.891 & \underline{0.924} & 0.912 & \underline{0.914} & \textbf{0.916} & 0.918 & \textbf{0.925}* & \textbf{0.924} & \textbf{0.921} & \underline{0.962} \\
\texttt{AIM}       & \textbf{0.599}* & \underline{0.715} & \underline{0.659}& \textbf{0.646} & \underline{0.676} & \textbf{0.840} & 0.935 & \textbf{0.855}* & \textbf{0.881}* & \underline{0.921} & \textbf{0.896}* & \underline{1.009} & \underline{0.930} & \underline{0.922}  & \textbf{0.930}* \\
\texttt{PrivMRF}   & 0.972 & 1.023 & 1.081 & 0.996 &  0.946 & \underline{0.857}  & \textbf{0.908}* & \underline{0.864} & 0.934 & 0.929 & 0.990& 1.125 & 0.968  & 1.006 & 1.046 \\
\texttt{GEM}       & 1.015&1.025 & 1.159 & 1.046 & 1.023 & 1.058 & 1.087 & 1.097 & 1.208 & 1.039 & \underline{0.914} & 1.020 & 0.950 & 0.933 & 0.974\\
\texttt{DP-MERF}   & 0.837 & 1.024 & 0.907 & 0.835 & 0.983 & 0.977 & 1.268 & 1.057 & 1.195 & 0.989 & 0.917 & 1.145 & 0.931 & 0.934 & 0.962 \\
\texttt{PrivBayes} & 0.736 & 0.776 & 0.763 & 0.772 & 0.729 & 1.029  & 1.315  & 1.042 & 1.083 & 1.047 & 0.999 & 1.174 & 2.153 & 1.190 &1.032 \\
\cmidrule[0.5pt]{1-16}
\textbf{Ground Truth}       &0.380 & 0.437 & 0.434 & 0.427 & 0.587 & 0.697 & 0.715 &0.703 & 0.725 & 0.862 & 0.145 & 0.222 & 0.154 & 0.181 & 0.900 \\
\bottomrule
\end{tabular}
}
\end{table}

\end{document}